\documentclass{article}

 \usepackage[main, final]{neurips_2026}

\usepackage[utf8]{inputenc} % allow utf-8 input
\usepackage[T1]{fontenc}    % use 8-bit T1 fonts
\usepackage{hyperref}       % hyperlinks
\usepackage{url}            % simple URL typesetting
\usepackage{booktabs}       % professional-quality tables
\usepackage{amsfonts}       % blackboard math symbols
\usepackage{nicefrac}       % compact symbols for 1/2, etc.
\usepackage{microtype}      % microtypography
\usepackage{xcolor}         % colors

\usepackage{url}
\usepackage{algorithm}
\usepackage{algorithmic}

\usepackage{amsmath}
\usepackage{amssymb}
\usepackage{mathtools}
\usepackage{amsthm}

\usepackage[textsize=tiny]{todonotes}

\usepackage{comment}
\usepackage{xspace}
\usepackage{enumitem}

\usepackage{amsfonts}
\usepackage{bm}
\usepackage{nicefrac}
\usepackage{dsfont}
\definecolor{mycyan}{RGB}{230, 255, 255}
\usepackage{mdframed}
\definecolor{theoremcolor}{rgb}{0.96, 0.96, 0.96}
\definecolor{examplecolor}{rgb}{1, 1, 1.0}
\mdfsetup{
    backgroundcolor=mycyan,
    linewidth=0pt,
}\newmdtheoremenv[linewidth=0pt,innerleftmargin=6pt,innerrightmargin=6pt]{definition}{Definition}
\newmdtheoremenv[linewidth=0pt,innerleftmargin=6pt,innerrightmargin=6pt]{proposition}{Proposition}
\newmdtheoremenv[linewidth=0pt,innerleftmargin=0pt,innerrightmargin=0pt,backgroundcolor=examplecolor]{example}{Example}
\newmdtheoremenv[linewidth=0pt,innerleftmargin=6pt,innerrightmargin=6pt]{corollary}{Corollary}
\newmdtheoremenv[linewidth=0pt,innerleftmargin=6pt,innerrightmargin=6pt]{theorem}{Theorem}
\newmdtheoremenv[linewidth=0pt,innerleftmargin=6pt,innerrightmargin=6pt]{lemma}{Lemma}
\newmdtheoremenv[linewidth=0pt,innerleftmargin=6pt,innerrightmargin=6pt]{remark}{Remark}

\usepackage{booktabs}
\usepackage{makecell}
\usepackage{multirow}
\usepackage{arydshln}

\usepackage{enumitem}
\setlist[enumerate]{
leftmargin=*}

\makeatletter
\def\adl@drawiv#1#2#3{%
        \hskip.5\tabcolsep
        \xleaders#3{#2.5\@tempdimb #1{1}#2.5\@tempdimb}%
                #2\z@ plus1fil minus1fil\relax
        \hskip.5\tabcolsep}
\newcommand{\cdashlinelr}[1]{%
  \noalign{\vskip 2pt
           \global\let\@dashdrawstore\adl@draw
           \global\let\adl@draw\adl@drawiv}
  \cdashline{#1}[.4pt/2pt]
  \noalign{\global\let\adl@draw\@dashdrawstore
           \vskip 2pt}}
\makeatother

\usepackage{tikz}
\usetikzlibrary{arrows.meta,calc,positioning,decorations.pathreplacing}

\definecolor{set10-red}{HTML}{e41a1c}
\definecolor{set10-blue}{HTML}{377eb8}
\definecolor{set10-green}{HTML}{4daf4a}

\newcommand{\kk}{{\bm{k}}}
\newcommand{\vv}{{\bm{v}}}
\newcommand{\K}{{\bm{K}}}
\newcommand{\V}{{\bm{V}}}

\title{Sparse Attention as Compact Kernel Regression}

\author{%
  Saul Santos$^{1, 2, 5}$,
  Nuno Gonçalves$^{1, 2, 3}$,
  Daniel C. McNamee$^{4}$, \\[0.25em]
  \bf Marcos Treviso$^{1, 2, 5}$,
  André F. T. Martins$^{1, 2, 5}$ \\[0.5em]
  $^{1}$Instituto Superior Técnico, Universidade de Lisboa $^{2}$Instituto de Telecomunicações \\
  $^{3}$Carnegie Mellon University 
  $^{4}$Champalimaud Research
  $^{5}$Gandara AI
}
\begin{document}

\maketitle

\begin{abstract}
Recent work has revealed a link between self-attention mechanisms in transformers and test-time kernel regression via the Nadaraya-Watson estimator, with standard softmax attention corresponding to a Gaussian kernel. 
However, a kernel-theoretic understanding of \textit{sparse} attention mechanisms is currently missing. 
In this paper, we establish a formal correspondence between sparse attention and \textit{compact} (bounded support) kernels. We show that normalized ReLU and sparsemax attention arise from Epanechnikov kernel regression under fixed and adaptive normalizations, respectively. 
More generally, we demonstrate that widely used kernels in nonparametric density estimation---including Epanechnikov, biweight, and triweight---correspond to $\alpha$-entmax attention with $\alpha = 1 + \frac{1}{n}$ for $n \in \mathbb{N}$, while the softmax/Gaussian relationship emerges in the limit $n \to \infty$. 
This unified perspective explains how sparsity naturally emerges from kernel design and provides principled alternatives to heuristic top-$k$ attention and other associative memory mechanisms. 
Experiments with a kernel-regression-based variant of transformers---Memory Mosaics---show that kernel-based sparse attention achieves competitive performance on language modeling, in-context learning, and length generalization tasks, offering a principled framework for designing  attention mechanisms.\footnote{Correspondence to \texttt{saul.r.santos@tecnico.ulisboa.pt}}.
\end{abstract}

\section{Introduction} 
Self-attention is a central component of transformer architectures, enabling models to capture long-range dependencies within sequences \citep{vaswani2017attention}. A growing body of work has established connections between attention-based in-context learning and variants of associative memories, including Hopfield networks \citep{ramsauer2020hopfield}, fast weight programmers \citep{schlag2021linear}, mesa-optimization \citep{von2023transformers}, and various forms of test-time regression and memory-based computation \citep[\textit{inter alia}]{wang2025testtime,behrouz2024titans,sun2025learning,behrouz2025s}. 
A key observation in several of these works is that standard softmax-based attention can be interpreted as \textit{Gaussian kernel regression} on the key-value cache---this has motivated new transformer-related architectures such as \textit{Memory Mosaics} \citep{zhang2025memory,zhang2025memory2}. 
However, \textit{is the Gaussian kernel---and softmax attention---the best design choice}? 
It has been pointed out that the dense weights produced by softmax attention mechanisms may cause attention dispersion \citep{velickovic2025softmax}---a phenomenon where many weak contributions from irrelevant positions dilute the influence of truly relevant ones. 
Evidence from brain research (see \S\ref{sec:related_work} for details) suggests that neocortical memory retrieval is mediated by a relatively small set of well-separated memories, rather than weakly integrating over an increasingly large number of memories~\cite{tulving2002episodic,gershman2025keyvalue}. 

The observations above provide strong motivation for \textit{sparse} attention alternatives, such as top-$k$ softmax \citep{gupta-etal-2021-memory}, sparsemax \citep{martins2016softmax}, and $\alpha$-entmax~\citep{peters2019sparse}. 
The last two provide continuous, fully differentiable mappings from scores to attention weights while maintaining exact sparsity (Figure~\ref{fig:overview}), enabling transformers to learn attention sparsity adaptively, rather than relying on a fixed truncation level.
However, unlike the softmax case, a theoretical test-time regression understanding of sparse attention is currently missing.

\begin{figure*}[t]
    \centering
    \includegraphics[width=0.32\linewidth]{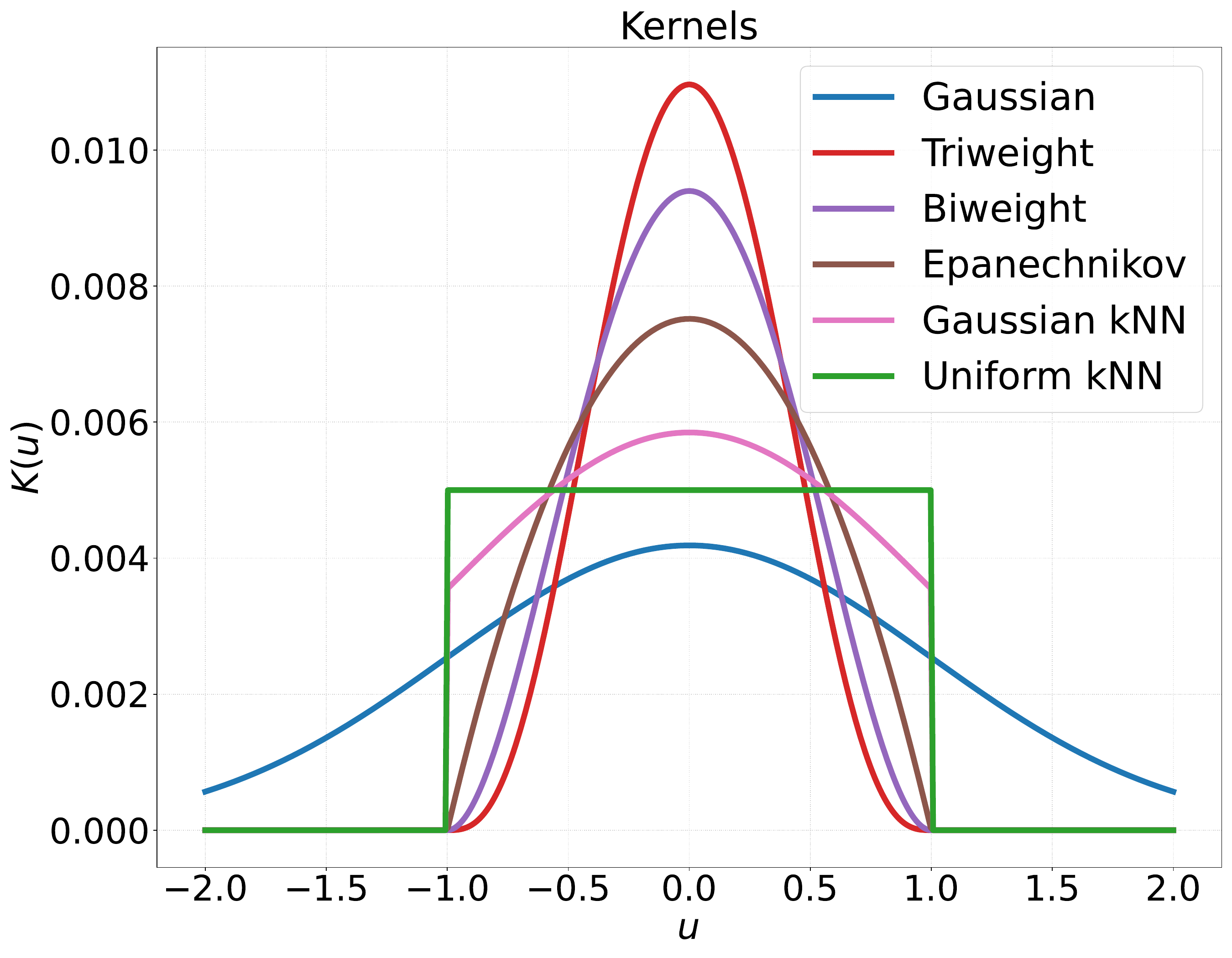}
    \hfill
    \includegraphics[width=0.63\linewidth]{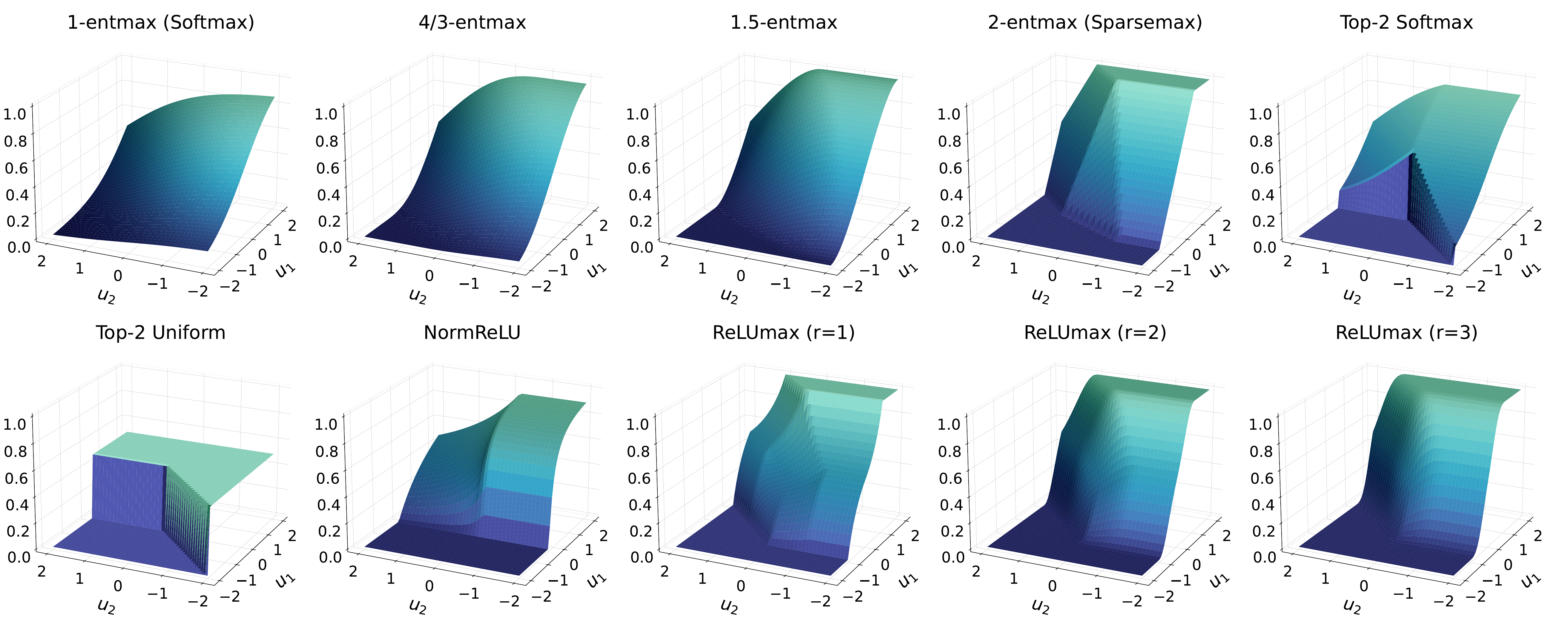}
    \caption{\textbf{Comparison between classical kernel functions and their corresponding attention activations.}
    The left plot shows normalized one-dimensional kernel functions commonly used in kernel regression evaluated on an input $\bm{u} \in \mathbb{R}^n$. 
    The right panel shows different attention transformations applied to a 3D score vector $\bm{u} = [0, u_1, u_2]$. 
    Each transformation allows an interpretation with corresponding kernels shown on the left: 
    softmax ($\alpha=1$) corresponds to the Gaussian kernel; 
    normalized ReLU, ReLUmax, and sparsemax ($\alpha=2)$ are derived from Epanechnikov kernel variants ($r=1$); 
  The 4/3- and 1.5-entmax transformations correspond to the biweight ($r=2$) and triweight ($r=3$) kernels, respectively, while 1- and 2-ReLUmax follow analogous patterns.
    top-$k$ uniform mirrors the top-$k$ uniform kernel, while top-$k$ softmax can be viewed as a truncated Gaussian kernel. }
   \label{fig:overview}
\end{figure*}

Our work aims to fill that gap by providing a unified kernel regression perspective on both dense and sparse attention mechanisms. We show that a range of attention transformations---including softmax, sparsemax, $\alpha$-entmax, top-$k$ uniform, top-$k$ softmax, normalized ReLU, and a new transformation, ReLUmax---can be interpreted as Nadaraya-Watson estimators with different kernel functions. This viewpoint not only explains how sparsity naturally emerges from kernel design but also offers a principled framework for developing attention mechanisms with controlled sparsity. 
Our main contributions are: %
\footnote{Our code is available in \url{https://github.com/deep-spin/sparse_kernel_regression}.}

\begin{enumerate}
\item \textbf{We establish that sparsemax attention corresponds to auto-normalized Epanechnikov kernel regression} through an adaptive bandwidth, providing a rigorous link between sparse attention and compact support kernels.

\item \textbf{More generally, we show that $\alpha$-entmax attention with $\alpha>1$ corresponds to kernel regression with compact support kernels} of the form $K(\bm{u}) \propto [1 - \|\bm{u}\|^2]_+^r$ where $r = (\alpha-1)^{-1}$, encompassing widely used kernels such as Epanechnikov, biweight, and triweight.

\item \textbf{We unify top-$k$ and fixed-normalization sparse attention within the same kernel regression framework}: top-$k$ softmax corresponds to Gaussian kernel $k$-nearest-neighbor (kNN) regression, 
normalized ReLU corresponds to fixed-bandwidth Epanechnikov regression.
\item \textbf{We propose a new transformation, $r$-ReLUmax}, which corresponds to a max-anchored polynomial rectified variant.

\item \textbf{We conduct experiments demonstrating the benefits of compact kernels in Memory Mosaics}, a strong transformer-like architecture tied to kernel regression, achieving superior performance on simple language modeling, in-context learning, and length generalization tasks.
\end{enumerate}

\section{Background}

%\subsection{Attention Mechanisms}

%The core of modern attention mechanisms is a transformation $\pi: \mathbb{R}^n \rightarrow \triangle_n$ that maps a vector of scores to a probability distribution over the simplex $\triangle_n := \{\bm{p} \in \mathbb{R}^n : \bm{p} \geq \mathbf{0}, \mathbf{1}^\top \bm{p} = 1\}$. Given query $\bm{q} \in \mathbb{R}^d$, keys $\K = [\kk_1, ..., \kk_n]^\top \in \mathbb{R}^{n \times d}$, and values $\V = [\vv_1, ..., \vv_n]^\top \in \mathbb{R}^{n \times d}$, the attention mechanism computes
%\begin{align}
%\sum_{i=1}^n \pi_i(\K\bm{q}) \vv_i = \V^\top \pi(\K \bm{q}).
%\end{align}
%The choice of transformation $\pi$ determines the behavior of the attention mechanism. The most common choice is the \textit{softmax transformation}, $\text{softmax}(\bm{z})_i = \frac{\exp(z_i)}{\sum_{j=1}^n \exp(z_j)}$. 

%Softmax produces dense attention weights, with all positions receiving nonzero probabilities. 
%For long contexts, this can lead to attention dispersion \citep{velickovic2025softmax} and makes token representations become less distinguishable~\citep{barbero2024transformers}, motivating  sparse alternatives. 

\subsection{Sparsemax and $\alpha$-entmax} \label{subsec:entmax_transormation}

A generalization of softmax is given by the $\alpha$-entmax transformation~\citep{peters2019sparse}, which has the form
\begin{align}
\label{eq:entmax}
\alpha\text{-}\text{entmax}(\bm{z})
=
\left[(\alpha - 1)\bm{z} - \tau \mathbf{1}\right]_+^{\frac{1}{\alpha - 1}},
\end{align} 
where $[x]_+ = \max(0,x)$ is the ReLU operator and $\tau$ is chosen such that $\sum_i p_i = 1$.
 This transformation arises as the solution of the optimization problem $
 \alpha\text{-}\text{entmax}(\bm{z})
 \;:=\;
 \arg\max_{\bm{p} \in \triangle_n}
 \;\bm{z}^\top \bm{p} + H_\alpha(\bm{p})$,
 where $H_\alpha(\bm{p}) = \frac{1}{\alpha(\alpha-1)}\left(1 - \sum_i p_i^\alpha\right)$
 is the Tsallis $\alpha$-entropy \citep{tsallis1988possible}. 
For $\alpha > 1$, $\alpha$-entmax produces sparse probability distributions, with the degree of sparsity increasing as $\alpha$ grows. The limit $\alpha \to 1$ recovers softmax, while the special case $\alpha = 2$ yields \textit{sparsemax}~\citep{martins2016softmax}.

\subsection{Attention and Kernel Regression}
\label{subsec:attention_and_kernel_methods}

Attention mechanisms can be interpreted through the lens of nonparametric kernel regression  \citep[\S1.5]{hardle1990applied,tsybakov2008nonparametric}.  
In this view, attention computes a locally weighted average of values, with the weights determined by a similarity-dependent kernel centered at the query. 

Formally, given data points $\{(\kk_i, \vv_i)\}_{i=1}^{n-1}$ and an input query $\bm{q} \equiv \kk_n$, the \textit{Nadaraya-Watson estimator} \citep{nadaraya, watson} predicts the next target value as
\begin{align} \label{eq:kernel}
\hat{\vv}_n(\bm{q}) = \mathbb{E}[\V_n | \bm{q}]
=
\sum_{i=1}^{n-1} \frac{ K_h(\kk_i - \bm{q})}
     {\sum_{j=1}^{n-1} K_h(\kk_j - \bm{q})}\vv_i,
\end{align}
where $K_h(\bm{u}) = \frac{1}{h^d} K\!\left(\frac{\bm{u}}{h}\right)$ is a kernel on $\mathbb{R}^d$ with bandwidth $h$, and $\int_{\mathbb{R}^d} K(\bm{u}) d\bm{u} = 1$.\footnote{Throughout, we use ``kernel'' in the classical sense of kernel density estimation and kernel regression. This should not be confused with positive-definite kernels in the RKHS sense.} 
The kernel controls how strongly each data point $(\kk_i, \vv_i)$ contributes to the prediction as a function of the distance between $\kk_i$ and the query $\bm{q}$.

\paragraph{Softmax attention as Gaussian kernel regression.}
Let $\bm{K} = [\kk_1, ..., \kk_{n-1}]^\top \in \mathbb{R}^{(n-1) \times d}$
denote the keys, $\bm{q} \equiv \bm{k}_n$ denote the query, and $\bm{V} = [\vv_1, ..., \vv_n]^\top \in
\mathbb{R}^{n \times d}$ the  values. For clarity,
we consider the setting in which the keys are $\ell_2$-normalized,
$\|\bm{q}\| = \|\kk_i\| = 1$, which allows us to directly relate dot-product
attention to Euclidean distances.

Consider the Gaussian kernel
$K(\bm{u}) \propto \exp(-\tfrac{1}{2}\|\bm{u}\|^2)$.
Using the identity
$\tfrac{1}{2}\|\kk_i - \bm{q}\|^2 = 1 - \kk_i^\top \bm{q}$,
the Nadaraya-Watson estimator~\eqref{eq:kernel} yields
\begin{align}
\label{eq:update}
\hat{\vv}_n(\bm{q})
=
\sum_{i=1}^{n-1}
\text{softmax}_i\!\left(\frac{\K\bm{q}}{h^2}\right)
\vv_i .
\end{align}
This expression recovers standard softmax attention \citep{vaswani2017attention}, with the kernel
bandwidth $h$ corresponding to the \emph{square root} of the softmax
temperature. In transformers, the conventional scaling
$\kk_i^\top \bm{q} / \sqrt{d}$ therefore corresponds to choosing
$h^2 = \sqrt{d}$, i.e., an effective bandwidth $h = d^{1/4}$. However, softmax produces dense attention weights, with all positions receiving nonzero probabilities. For long contexts, this can lead to attention dispersion \citep{velickovic2025softmax} and makes token representations become less distinguishable~\citep{barbero2024transformers}, motivating  sparse alternatives.

Choosing an appropriate bandwidth is central in kernel regression
theory~\citep{tsybakov2008nonparametric}. While vanilla transformers fix this
choice via the $\sqrt{d}$ scaling, recent works adapt the effective
bandwidth as a function of the sequence length $n$ to control attention
dispersion and long-context behavior~\citep{nakanishi2025scalable,
vasylenko2025longcontextgeneralizationsparseattention,
zhang2025memory2}.
Furthermore, this kernel regression view suggests a natural generalization, raising the following question: \textit{``what if we use a different kernel than the Gaussian?''}.

\paragraph{Beyond Gaussian kernels.}
Classical kernel regression is not restricted to Gaussian kernels. A wide range of kernels has been studied, differing in smoothness and support properties. Of particular interest are kernels with \emph{compact support}, which assign zero weight to points beyond a certain distance and therefore induce locality and sparsity.

Common examples include the \textit{Epanechnikov kernel}
$
K(\bm{u}) \propto [1 - \|\bm{u}\|^2]_+$ \citep{epanechnikov},  
which has optimal properties in terms of minimizing mean integrated squared error (MISE), 
as well as its higher-order polynomial variants, including the \textit{biweight}
$
K(\bm{u}) \propto [1 - \|\bm{u}\|^2]_+^2
$
and \textit{triweight} kernels
$
K(\bm{u}) \propto [1 - \|\bm{u}\|^2]_+^3
$~\citep{scott2015multivariate}. These kernels have restricted support to a bounded neighborhood around the query, as illustrated in Figure~\ref{fig:overview}.

Another way to enforce compact support is to truncate the Gaussian or uniform kernel to the $k$ nearest neighbors:
\begin{equation}
\label{eq:topk-kernel}
\begin{aligned}
K_{\text{top-}k}(\bm{u}) &\propto
\begin{cases}
\exp\!\big(-\tfrac{1}{2}\|\bm{u}\|^2 / h^2\big), & \|\bm{u}\| \le \delta_{\max},\\
0, & \text{otherwise,}
\end{cases}
\end{aligned}
\quad
\begin{aligned}
K_{\text{uniform}}(\bm{u}) &\propto
\begin{cases}
1, & \|\bm{u}\| \le \delta_{\max},\\
0, & \text{otherwise.}
\end{cases}
\end{aligned}
\end{equation}
where the threshold $\delta_{\max}$ is implicitly determined by the set $\mathcal{N}_k(\bm{q})$ of the $k$ nearest neighbors of $\bm{q}$ in feature space, i.e., $\|\kk_i - \bm{q}\| \le \delta_{\max}$ if and only if $\kk_i \in \mathcal{N}_k(\bm{q})$.

Epanechnikov, biweight, triweight, top-$k$ Gaussian, and uniform kernels are all examples of kernels with compact support. As we show in \S\ref{sec:sparse_attention}, these classical kernel choices give rise to different sparse attention mechanisms, providing a unified kernel-based perspective on locality and sparsity in attention. We illustrate these kernels in Figure~\ref{fig:overview}.

\subsection{Memory Mosaics}

\textit{Memory Mosaics} \citep{zhang2025memory} is an associative memory variant of the transformer architecture designed to perform prediction tasks using kernel regression; this architecture has recently shown to exhibit good performance at reasonable scales \citep{zhang2025memory2}. We base our work and experiments on this architecture as it \textbf{explicitly implements a Nadaraya-Watson estimator}, as described in \eqref{eq:update}, unlike vanilla transformers. Specifically, in Memory Mosaics, the query and key are \textit{the same} by design, aligning exactly with kernel regression: $\kk_n$ serves both as a query at the $n\textsuperscript{th}$ time step to determine the weighting of past $(n-1)$ stored values, and as a key at subsequent time steps. Each memory unit contains a trainable feature extractor that computes keys and values from input vectors $\bm{x}_i \in \mathbb{R}^{d_z}$. 
\begin{equation}
\setlength{\jot}{2pt} % tighten row spacing (optional)
\begin{aligned}
\text{Keys:} \quad 
& \kk_n = \varphi(\bm{x}_n, \bm{x}_{n-1}, \dots; \Theta)
& \quad
\tilde{\kk}_n = W_{\varphi}\bm{z}_n
& \quad
\kk_n = \mathrm{Norm}(\tilde{\kk}_n + \lambda_{\varphi} \overline{\kk}_{n-1}) \\
\text{Values:} \quad 
& \vv_n = \psi(\bm{x}_{n+1}, \bm{x}_n, \bm{x}_{n-1}, \dots; \Theta)
& \quad
\tilde{\vv}_n = W_{\psi}\bm{z}_n
& \quad
\vv_n = \mathrm{Norm}(\tilde{\vv}_n + \lambda_{\psi} \tilde{\vv}_{n+1})
\end{aligned}
\label{eq:combined}
\end{equation}
Here, $\kk_n, \vv_n \in \mathbb{R}^{d}$ denote the key and value representations, and $\Theta$ represents learnable parameters. 
The key $\kk_n$ depends only on current and past embeddings $(\bm{x}_i)_{i \le n}$, whereas the value $\vv_n$ incorporates a one-step lookahead to capture predictive structure. 
The projections use linear mappings $W_\varphi, W_\psi \in \mathbb{R}^{d \times d_z}$ followed by leaky averaging with coefficients $\lambda_\varphi, \lambda_\psi$, and normalization. 
The term $\overline{\kk}_{n-1}$ denotes the running average of past keys.

In the Memory Mosaics architecture, each standard transformer block is replaced by two memories: a \textit{contextual} memory unit, which applies attention via kernel regression using~\eqref{eq:update} and~\eqref{eq:combined}, and a \textit{persistent} memory module. 
The persistent memory, inspired by earlier memory-augmented networks~\citep{sukhbaatar}, consists of learned memory slots that store information across sequences, providing a form of long-term associative memory. Conceptually, it functions similarly to self-attention over fixed memory slots and can replace a standard feed-forward network by offering additional capacity and nonlinear processing. The attention mechanism excludes the main diagonal as opposed to vanilla transformers, making use of \eqref{eq:update} where the sum runs only up to $n-1$ rather than $n$, which is why values can ``peek'' one step into the future in \eqref{eq:combined} to compensate. In this work, we focus on adapting the contextual memory unit to incorporate sparsity via different kernel choices while retaining the Nadaraya-Watson estimation principle.

\section{Sparse Attention via Compact Kernels}
\label{sec:sparse_attention}

We now formalize the connection between sparse attention mechanisms and kernel regression with compact support kernels. 
Starting from the kernel view of attention described in \S\ref{subsec:attention_and_kernel_methods}, we investigate how replacing the Gaussian kernel implicit in softmax attention with compact-support kernels naturally induces locality and sparsity.

\paragraph{Normalized ReLU and Epanechnikov kernel.}

Consider the Epanechnikov kernel mentioned earlier in \S\ref{subsec:attention_and_kernel_methods}, 
a compactly supported kernel widely used in kernel density estimation.
Replacing the Gaussian kernel with the Epanechnikov kernel, yields the following query-key interaction and normalized attention rule:
\begin{equation}
\begin{aligned}
\begin{aligned}
K(\kk_i - \bm{q})
&= \big[\kk_i^\top \bm{q}/\gamma + b\big]_+ 
\end{aligned}
\quad \Longrightarrow \quad
\begin{aligned}
\hat{\vv}_n(\bm{q})
&= \sum_{i=1}^{n-1}
\frac{\big[\kk_i^\top \bm{q}/\gamma + b\big]_+}
     {\sum_{j=1}^{n-1} \big[\kk_j^\top \bm{q}/\gamma + b\big]_+}\,
\vv_i .
\end{aligned}
\end{aligned}
\label{eq:normrelu_attention}
\end{equation}
with $\gamma = h^2/2$ and $b = 1 - 2/h^2$. Thus, \textit{normalized ReLU attention is equivalent to Epanechnikov kernel regression}, with sparsity arising from the ReLU truncation. 
This connection was also noted by \cite{hoover2025dense}, who used it to derive an Hopfield-like energy function.
A limitation of the normalized ReLU \eqref{eq:normrelu_attention} is that the attention weights are undefined when all pre-activations are negative, leading to $\frac{0}{0}$. 
Next, we show two principled ways to avoid this degeneracy: (i) \emph{auto-normalization} (adaptive bandwidth) leading to sparsemax and $\alpha$-entmax, and (ii) \emph{anchoring the support} at nearest keys, leading to ReLUmax.

\subsection{Auto-normalization (Adaptive Bandwidth)}

An alternative to the normalization \eqref{eq:normrelu_attention} is to choose a \textit{data-dependent bandwidth} $h$ which ensures the denominator in the Nadaraya-Watson estimator is a constant---we call this \textit{auto-normalization}. 
We show next that this procedure yields attention transformations such as sparsemax and $\alpha$-entmax.

We begin with sparsemax and show how it arises from auto-normalized Epanechnikov regression. 

Since we have normalized queries and keys, we can rewrite the squared distance as
$
\frac{1}{2}\|\kk_i - \bm{q}\|^2 = 1 - \kk_i^\top \bm{q}
$. 
Then, introducing a temperature $\gamma$, the denominator in \eqref{eq:kernel} can be rewritten as
\begin{align}
\sum_{i=1}^{n-1}\left[1 - \frac{\|\kk_i - \bm{q}\|^2}{h^2}\right]_+
&= \left(\frac{2\gamma}{h^2}\right)
\sum_{i=1}^{n-1}
\left[\frac{h^2}{2\gamma} - \frac{1}{\gamma} + \frac{\kk_i^\top \bm{q} }{\gamma}\right]_+^.
\label{eq:epanechnikov_rewrite}
\end{align}
% with $\gamma = h^2/2$ and $b = 1 - 2/h^2$.  
% That is, the induced kernel takes the form of a rectified polynomial kernel of order $r$.
% encompassing classical kernels such as the Epanechikov ($r=1$), biweight ($r=2$), and triweight ($r=3$). 
In Nadaraya-Watson estimation, the denominator enforces that the weights sum to one. 
An alternative is to absorb this normalization into the kernel scale such that, rather than fixing $h$, we choose it 
so that the rectified responses sum to one. 
Concretely,
\begin{align}
\begin{aligned}
\sum_{i=1}^{n-1}
\left[\frac{\kk_i^\top \bm{q}}{\gamma} - \underbrace{\left(\frac{1}{\gamma} - \frac{h^2}{2\gamma}\right)}_{:= \tau}  \right]_+ = 1 
\end{aligned}
\quad \Longrightarrow \quad
\begin{aligned}
\hat{\vv}_n(\bm{q})
= \sum_{i=1}^{n-1}
\text{sparsemax}_i\!\left(\frac{\K\bm{q}}{\gamma}\right)\vv_i,
\end{aligned}
\label{eq:sparsemax_kernel_regression}
\end{align}
showing that Epanechnikov kernel regression with auto-normalization is equivalent to sparsemax attention. In this construction the effective bandwidth is therefore \emph{adaptive} $h = \sqrt{2 - 2\gamma\tau}$,
where $\tau$ (hence $h$) is determined by the attention scores via \eqref{eq:epanechnikov_rewrite}, which depends only on the keys (but not the values). 
This contrasts with common bandwidth schedules that depend only on sequence length~\citep{zhang2025memory2}.
It is also interesting to compare \eqref{eq:sparsemax_kernel_regression} with \eqref{eq:update}: in the softmax attention induced by the Gaussian kernel, we simply have $h = \sqrt{\gamma}$. 

\paragraph{From sparsemax to $\alpha$-entmax via compact kernels.}
Instead of the linear truncation used by sparsemax, $\alpha$-entmax introduces a higher-order rectification controlled by a parameter $\alpha>1$, producing attention weights with an implicitly defined threshold $\tau$. On the kernel side, this corresponds to rectified polynomial kernels of the form $K_h(\bm{u}) \propto [1-\|\bm{u}\|^2/h^2]_+^r$, where $r = 1/(\alpha-1)$. This family recovers Epanechnikov ($r=1$), biweight ($r=2$), and triweight ($r=3$) as special cases. Proposition~\ref{prop:entmax_rectified_poly} formalizes this correspondence, and the proof is given in Appendix~\ref{proof:main_proof}.
\begin{proposition}\label{prop:entmax_rectified_poly}
Let $r\ge1$ and $\alpha=1+\frac{1}{r}$.  
Let $\bm{K}=[\kk_1,...,\kk_{n-1}]^\top\in\mathbb{R}^{(n-1)\times d}$ be normalized keys, $\bm{q} \equiv \bm{k}_n\in\mathbb{R}^d$ the query, and 
$\bm{V}=[\vv_1,...,\vv_n]^\top$ the  values.  
The $\alpha$-entmax attention with temperature $\gamma$ satisfies
\begin{align} \label{eq:entmax_proposition_1}
\sum_{i=1}^{n-1} \alpha\text{-entmax}_i\!\left(\frac{\bm{K}\bm{q}}{\gamma}\right)\vv_i
=
\frac{\sum_{i=1}^{n-1} K_h(\kk_i - \bm{q})\,\vv_i}
{\sum_{i=1}^{n-1} K_h(\kk_i - \bm{q})},
\end{align}
where $K_h(\bm{u}) \propto \left[1 - \frac{\|\bm{u}\|^2}{h^2}\right]_+^r$ and $h$ is implicitly defined by $h = \sqrt{2 - 2r\gamma \tau}$. 
\end{proposition}
Proposition~\ref{prop:entmax_rectified_poly} yields a continuous family of compact-support kernel attention mechanisms indexed by $\alpha$. 
In particular, $r=2$ corresponds to $1.5$-entmax and the biweight kernel, while $r=3$ corresponds to $\tfrac{4}{3}$-entmax and the triweight kernel, recovering compact-kernel associative memories studied in Hopfield networks and related energy-based models~\citep{wu2024stanhop,santos2024sparse, santos2024hopfieldfenchelyoungnetworksunifiedframework}. Figure~\ref{fig:overview} summarizes these correspondences between classical compact kernels and sparse attention transformations.

The rectified polynomial family is only one way to obtain compact support. A second, complementary route is to \emph{anchor} sparsity directly on the nearest keys by restricting attention to a neighborhood defined by kNN selection, i.e., by truncating an underlying kernel to a data-dependent support set. 
We turn to these anchored compact kernels next.

\subsection{Anchored Compact Kernels}
\label{subsec:anchored_compact_kernels}

In this second route, instead of changing the kernel shape, we keep a base kernel and make its \emph{support set} data-adaptive, either by truncation (top-$k$) or by max-anchoring (ReLUmax). 

\paragraph{Top-$k$ average pooling (uniform-kNN).}
Consider the uniform kernel in \eqref{eq:topk-kernel}. 
If we choose $\delta_{\max}$ so that exactly $k$ keys satisfy $\|\bm{k}_i-\bm{q}\|\le \delta_{\max}$, then the support set is precisely the $k$-nearest-neighbor set $\mathcal{N}_k(\bm{q})$. Because all selected keys receive identical weight under the uniform kernel, the Nadaraya-Watson estimate reduces to $\hat{\bm{v}}(\bm{q}) = \frac{1}{k}\sum_{i \in \mathcal{N}_k(\bm{q})} \bm{v}_i$. Thus, top-$k$ average pooling is equivalent to uniform-kNN kernel regression with an adaptively chosen radius that admits exactly $k$ neighbors.

\paragraph{Top-$k$ softmax (truncated Gaussian).}  

A smoother, anchored alternative is to use a Gaussian kernel within the $k$-nearest-neighbor set and set all weights to zero outside it. This corresponds to the truncated kernel in \eqref{eq:topk-kernel}. 
Normalizing these weights yields
\begin{align}
\hat{\vv}(\bm{q})
    &= \sum_{i \in \mathcal{N}_k(\bm{q})}
       \frac{K_{\text{top-}k}(\kk_i - \bm{q})}
            {\sum_{j \in \mathcal{N}_k(\bm{q})}
                   K_{\text{top-}k}(\kk_j - \bm{q})}
       \;\vv_i .
\end{align}
This mechanism can be interpreted as Gaussian kernel regression with a data-adaptive compact support set, given by $\mathcal{N}_k(\bm{q})$, yielding a sparse kNN-style attention mechanism~\citep{gupta-etal-2021-memory}. 

\paragraph{$r$-ReLUmax.}
The top-$k$ constructions above enforce compact support by \emph{hard} truncation via kNN selection. 
We consider a new transformation,  $r$-ReLUmax, which anchors support \emph{softly} at the best-matching key while retaining a compact kernel and avoiding the empty-support degeneracy of normalized ReLU.
Let
% $\Delta := \min_j \|\kk_j - \bm{q}\|^2$
$m := \max_j \bm{k}_j^\top\bm{q}$ and $b > 0 \in \mathbb{R}$ be a hyperparameter. The $r$-ReLUmax kernel and the corresponding Nadaraya-Watson estimate are
\begin{align}
K_{\text{$r$-relumax}}(\bm{k}_i-\bm{q})
&\propto
\Big[
b + \frac{\bm{k}_i^\top \bm{q} - m}{h^2}
\Big]_+^r
\quad \Longrightarrow \quad
\hat{\bm{v}}(\bm{q})
=
\sum_{i=1}^{n-1}
\frac{
\left[
b + \frac{\bm{k}_i^\top \bm{q} - m}{h^2}
\right]_+^r
}{
\sum_{k=1}^{n-1}
\left[
b + \frac{\bm{k}_k^\top \bm{q} - m}{h^2}
\right]_+^r
}
\;\bm{v}_i .
\label{eq:r-relumax_kernel}
\end{align}
Importantly, since $b > 0$, in the worst case where all attention scores equal $m$, at least one rectified score will be strictly positive, and so the denominator in \eqref{eq:r-relumax_kernel} cannot be zero. Hence, ReLUmax avoids the $0/0$ degeneracy of normalized ReLU \eqref{eq:normrelu_attention}  while preserving compact support.\looseness=-1
Furthermore, we note that  \eqref{eq:r-relumax_kernel} yields an explicit characterization of the active set:
\begin{align}
K_{\text{$r$-relumax}}(\bm{k}_i-\bm{q})>0
\quad&\Longleftrightarrow\quad
b + \frac{\bm{k}_i^\top\bm{q} - m}{h^2} > 0
&\Longleftrightarrow\quad
\bm{k}_i^\top\bm{q} > m - b h^2.
\label{eq:relumax_margin_rule}
\end{align}
Thus, ReLUmax selects \emph{all keys within a fixed similarity margin of the maximum similarity}. This is fundamentally different from top-$k$ since the number of active keys adapts to the distribution of score gaps rather than being fixed.
In fact, this points to a \emph{top-$p$} viewpoint. 
Top-$k$ fixes the \emph{cardinality} of the support, whereas ReLUmax fixes a \emph{margin-to-maximum} threshold, with keys being kept whenever their score exceeds $m-bh^2$. 
Consequently, the retained set expands when many keys have scores close to the maximum (small gaps) and shrinks when the maximum is well separated (large gaps). 
In this sense, ReLUmax behaves like a continuous analogue of nucleus selection as it adaptively includes ``as many keys as needed'' near the top of the score distribution, but does so via a smooth, differentiable margin.

\section{Experiments}
\begin{figure*}[ht!]
    \centering
    \includegraphics[width=1\linewidth]{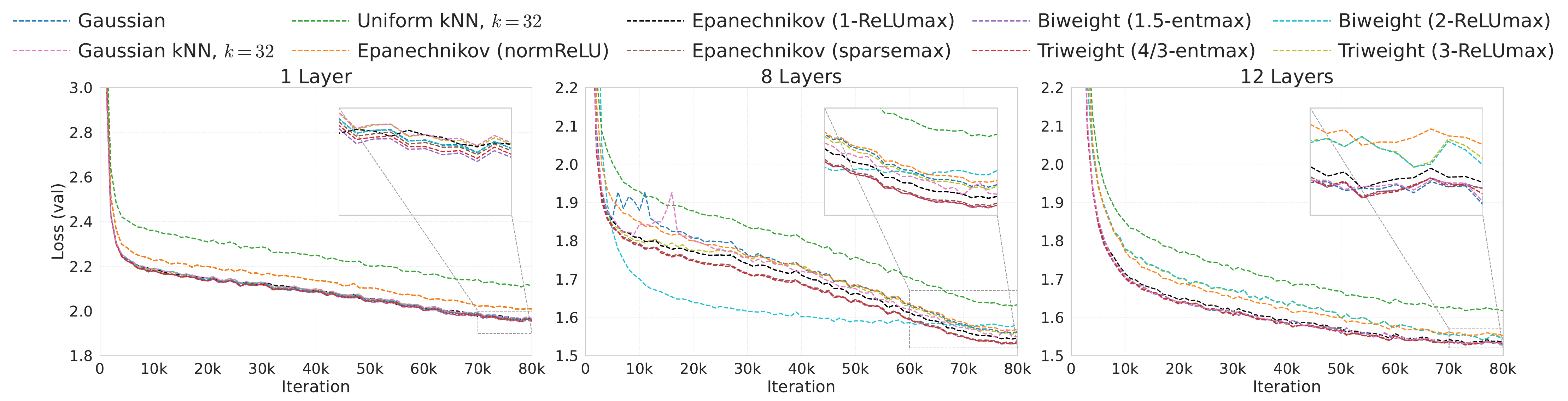}
    \caption{\textbf{Validation loss of Memory Mosaic models with different kernels.} Comparison on the BabiStories dataset for varying model depths. The horizontal axis shows the number of training iterations.}

    \label{fig:train_val}
\end{figure*}
We conduct three sets of experiments to validate our theoretical findings and demonstrate the practical benefits of our kernel-based attention mechanisms in Memory Mosaics.

% \paragraph{Why Memory Mosaics?}
% In summary, Memory Mosaics exactly matches the kernel-regression setting studied here. The difference from a causal transformer is small but important: standard causal self-attention attends over positions $i \le n$, whereas Memory Mosaics attend strictly over past memories $i<n$. 
% This exactly matches the Nadaraya-Watson estimator in Eq.~\eqref{eq:kernel}.
% % where the query $\bm{k}_n$ is evaluated against a stored dataset $\{(\bm{k}_i,\bm{v}_i)\}_{i<n}$ and is not one of its own neighbors. 
% Thus, Memory Mosaics provide a controlled testbed for isolating the effect of the kernel choice, complementary to prior transformer-based evidence that sparse attention mechanisms such as $\alpha$-entmax are effective~\citep{gonccalves2026adasplash}.

\subsection{Language Modeling}
\label{sec:language_modeling}
\paragraph{Setup.}
We replicate Memory Mosaics \citep{zhang2025memory} with our proposed  kernel-based sparse attention mechanisms 
% To evaluate the practical impact of different variants, w
and conduct language modeling experiments on small, self-contained narratives written in simple English, designed to be comprehensible to a young child. 
Specifically, we follow the methodology of  \cite{li2024tinystories} and \cite{zhang2025memory}, and evaluate on BabiStories.
To study the effect of different attention mechanisms, we replace the standard softmax attention used in Memory Mosaics \citep{zhang2025memory} with the family of kernel-based methods described in \S\ref{sec:sparse_attention}. 
Namely, we consider Gaussian (softmax) attention and compact rectified polynomial kernels: Epanechnikov ($r{=}1$) via normalized ReLU, sparsemax, and 1-, 2-, and 3-ReLUmax ($b{=}1$), and higher-order biweight ($r{=}2$) and triweight ($r{=}3$), which correspond to $\alpha$-entmax with $\alpha{=}1.5$ and $\alpha{=}\tfrac{4}{3}$, respectively. 
In addition to $r$-order rectified polynomial kernels, we evaluate discrete uniform and Gaussian kNN attention. 
Experimental details can be found in \S\ref{subsec:app_lm}. For 2-ReLUmax and 8 layers, a slightly lower learning rate was required due to training instabilities.

\paragraph{Discussion.}
Figure~\ref{fig:train_val} shows a comparison of models equipped with different kernel-based attention mechanisms across model depths in terms of validation loss. 
For shallow models (one layer), compactly supported kernels, corresponding to adaptive bandwidth, achieve lower validation loss than dense Gaussian attention and top-$k$ baselines.
As depth increases to 8 layers, this advantage remains, and 
rectified polynomial kernels with adaptive bandwidth outperform dense softmax attention and anchored top-$k$ variants across iterations. 
With 12 layers, the gap is narrower and, except normReLU and uniform kNN, all approaches achieve roughly the same validation loss.
Interestingly, while top-$k$ Gaussian attention can reach competitive validation loss at intermediate depths, it exhibits mild instability, likely due to the non-smoothness introduced by discrete neighbor selection. 
Among sparse mechanisms, auto-normalized attention mappings---corresponding to sparsemax and $\alpha$-entmax---consistently perform better than normReLU. While $r$-ReLUmax outperforms normReLU with 8 and 12 layers,
it slightly underperforms adaptive normalization methods. We use AdaSplash \citep{goncalves2025adasplash} for efficiency, with compact kernels, and discuss the efficiency of all methods in \S\ref{app:efficiency}.

\subsection{In-Context Learning}

\paragraph{Setup.} We also evaluate our models on an in-context learning task. To rigorously compare the in-context learning capabilities of different architectures, we adopt the RegBench benchmark \citep{akyürek2024incontextlanguagelearningarchitectures}, a diagnostic suite designed to probe sequence models’ ability to infer the generative processes of regular languages, which constructs random artificial languages defined by a PFA with sequences consisting of 10 to 20 strings sampled from the same PFA. 
We compare baseline transformers and kernelized variants of our sparse Memory Mosaic architectures, as well as several established baselines including S4 \citep{gu2022efficiently}, Gla \citep{yang2023gated}, RetNet \citep{sun2023retentive}, H3 \citep{fu2023hungry}, Mamba \citep{mamba}, and standard transformers \citep{vaswani2017attention}.
Our models match the sequence length and vocabulary size of \cite{akyürek2024incontextlanguagelearningarchitectures} and follow the same hyperparameter tuning strategy. 
Full experimental details are provided in \S\ref{app:icl_experimental_details}.
Each model is trained on datasets ranging from 1k to 40k examples and evaluated using per-token loss as a function of the generated token position for a test set of 1k examples.
\begin{table*}[t!]
\centering
\small
\setlength{\tabcolsep}{6pt}
\caption{\textbf{Performance comparison across varying training set sizes in the Regbench in-context learning benchmark.} Accuracy ($\uparrow$) measures the correct prediction rate, and TVD ($\downarrow$) measures divergence from the target distribution. We \textbf{bold} the best performing models.}
\renewcommand{\arraystretch}{1.2}
\resizebox{\textwidth}{!}{
\begin{tabular}{l
    *{6}{c c}
}
\toprule
\multirow{2}{*}{\textbf{Method}} & 
\multicolumn{2}{c}{\textbf{1k}} &
\multicolumn{2}{c}{\textbf{2.5k}} &
\multicolumn{2}{c}{\textbf{5k}} &
\multicolumn{2}{c}{\textbf{10k}} &
\multicolumn{2}{c}{\textbf{20k}} &
\multicolumn{2}{c}{\textbf{40k}} \\
\cmidrule(lr){2-3} \cmidrule(lr){4-5} \cmidrule(lr){6-7} 
\cmidrule(lr){8-9} \cmidrule(lr){10-11} \cmidrule(lr){12-13}
& Acc $\uparrow$ & TVD $\downarrow$
& Acc $\uparrow$ & TVD $\downarrow$
& Acc $\uparrow$ & TVD $\downarrow$
& Acc $\uparrow$ & TVD $\downarrow$
& Acc $\uparrow$ & TVD $\downarrow$
& Acc $\uparrow$ & TVD $\downarrow$ \\
\midrule
RetNet & 51.3 & 67.4 & 50.5 & 63.9 & 81.5 & 33.3 & 88.4   & 25.7 & 90.8 & 22.5 & 92.6& 18.9\\
H3 & 51.5 & 65.1 & 58.0 & 57.0 & 70.9 & 47.4 & 76.9 & 41.4 & 80.7 & 36.3& 87.7 & 25.4 \\
Gla & 51.8 & 65.3 & 51.1 & 62.2 & 82.1 & 31.0 & 90.2 & 24.5 & 91.2 & 21.2& 92.9 & 19.0 \\
S4 & 52.1 & 62.8 & 51.7 & 62.2 & 54.5 & 59.0 & 63.3 & 52.9 & 72.1 & 45.8 &82.5 & 33.8\\
Mamba & 69.4 & 51.0 & 79.4 & 41.0 & 88.7 & 30.6 & 90.5 &  28.4& 91.3 & 27.0 & 91.5 & 27.0 \\
Transformer & 54.5 & 62.6 & 92.4 & 20.5 & 92.7 & 18.9 &  94.8 & 15.7 & 95.1 & 15.0 & 95.5 & 13.9\\
\midrule
\multicolumn{13}{l}{\textbf{Memory Mosaics:}} \\
Gaussian & 93.9 & 34.1 & 93.7 & 20.8 & 94.6 & 17.7 & 95.0 & 15.9 & \textbf{95.4} & 15.3 & 95.4 & 14.2\\
Gaussian kNN, $k=16$  & 93.4 & 33.3 & 93.7 & 19.8 & 94.4 & 17.2 & 94.8 & 16.0 & 94.8 & 16.0 & 95.0 & 14.9 \\
Uniform kNN, $k=16$ &84.6&41.1&85.6&28.3&86.1&27.2&86.2&26.4&86.5&26.4&86.4&26.4 \\
Epanechnikov (normReLU)&84.3&43.6&87.0&29.6&90.1&24.6&91.6&21.7&93.1&19.4&93.3&18.6  \\
Epanechnikov (1-ReLUmax)& 93.8&36.8&94.2&20.4&94.3&17.8&94.9&\textbf{15.7}& \textbf{95.4}&\textbf{15.0} &95.4 &14.1 \\
Epanechnikov (sparsemax) & \textbf{94.2} & \textbf{32.5} & 94.0 & 19.3 & 94.4 & 17.1 & 94.9 & 15.9 & 95.0 & \textbf{15.0} &  95.4&14.0 \\
Biweight (2-ReLUmax)& 93.8 & 36.4 & 94.4 & 20.1&94.6& 17.2&95.1 &16.1 & 95.3 &15.1&95.1&14.8\\
Biweight (1.5-entmax) & 93.8 & 33.0 & \textbf{94.3} & 19.2 & 94.7 & 16.6 & \textbf{95.2} &  \textbf{15.7}& 95.0 & 15.8 & \textbf{95.5} & \textbf{13.8} \\
Triweight (3-ReLUmax) & 93.9 & 36.6&94.2&20.5&94.7&16.9&94.8 &16.4&95.1 &15.4&95.3 &14.3 \\
Triweight (4/3-entmax) & 93.9 & 32.9 & \textbf{94.3} & \textbf{19.0} & \textbf{94.8} & \textbf{16.2} & 94.9 & 16.1 & 95.1 & 15.3 & 95.3 & 14.4 \\
\bottomrule
\end{tabular}}
\label{tab:icl}
\end{table*}
\paragraph{Discussion.}
Table~\ref{tab:icl} compares models on RegBench using last-token accuracy and total variation distance (TVD). Among the Memory Mosaic variants, rectified polynomial kernels (Triweight, Biweight, and Epanechnikov) with $\alpha$-entmax consistently match or slightly outperform the Gaussian kernel in the low-data regime (1k–5k), achieving both higher accuracy and lower TVD. The adaptive bandwidth family ($\alpha$-entmax) performs best overall as compared to max-anchoring ($r$-ReLUmax). In contrast, normalized ReLU shows a clear degradation in both metrics, suggesting that simple thresholded normalization without adaptive support selection is insufficient in this setting. 
Fixed top-$k$ uniform attention underperforms across all data scales, highlighting the limitations of hard, data-independent sparsity compared to kernel-based adaptive weighting. 
Overall, these results indicate that adaptive sparsity induced by compact kernels is most beneficial in data-scarce regimes.

\subsection{Length Generalization}

\paragraph{Setup.} We evaluate the length generalization capabilities of our Memory Mosaics using synthetic tasks, following the setup of \cite{vasylenko2025longcontextgeneralizationsparseattention}, such as MQMTAR (multi-query multi-token associative recall), sequence sorting, and reversing, which provide precise control over training and test sequence lengths. These tasks test the model’s ability to maintain attention on relevant, sparse subsets of tokens, and probe the model's ability to learn an underlying length-agnostic algorithm,
rather than memorizing patterns. 
Full experimental details are provided in \S\ref{app:synthetic_experimental_details}. We note that Memory Mosaics does not rely on positional embeddings. For completeness, we also evaluate a variant combining the Epanechnikov kernel with ALiBi positional bias \citep{press2022train}. As shown in \S\ref{app:alibi}, this modification leads to a consistent degradation in performance, motivating our decision to omit positional embeddings in the final model.
 
\paragraph{Discussion.}
The results in Table~\ref{tab:LG_results} show that Gaussian attention performs well in-distribution but quickly deteriorates when extrapolating to long sequences. 
Top-$k$ uniform attention is too crude, failing to extrapolate for any of the tasks. 
Gaussian kNN mitigates some degradation by limiting attention to nearest neighbors, yet it still achieves poor length extrapolation, especially on very long sequences. 
In contrast, our rectified polynomial kernels ($\alpha$-entmax) exhibit strong and stable generalization, particularly on the MQMTAR task. Triweight and biweight kernels outperform Gaussian and top-$k$ mechanisms on MQMTAR, with triweight also leading on the reverse task, while preserving high accuracy over longer sequences. Epanechnikov-based attention mechanisms, specifically 1-ReLUmax and sparsemax, consistently achieve perfect in-distribution accuracy while maintaining superior length extrapolation performance, with sparsemax showing better generalization performance.
Finally, 1-ReLUmax outperforms normReLU (for $r=1$), benefiting from its max-anchoring operation, with improved stability at longer sequence lengths.
\begin{table*}[t]
\caption{Exact match accuracy on representative synthetic tasks. In-distribution results ($n=64$) and out-of-distribution performance at increasing sequence lengths are reported. Values show the mean across three seeds, with the maximum across seeds indicated in superscript. Best average performance is in \textbf{bold} and overall maximum performance is \underline{underlined}. $L$ is the number of layers.}
  \small
  \label{tab:LG_results}
  \centering
   \setlength{\tabcolsep}{2pt}
   \resizebox{\textwidth}{!}{
  \begin{tabular}{l ccccccc cc ccc}
    \toprule
    & \multicolumn{7}{c}{MQMTAR $(L=4)$} 
    & \multicolumn{2}{c}{Reverse $(L=4)$}
    & \multicolumn{3}{c}{Sort $(L=2)$} \\
    \cmidrule(lr){2-8} \cmidrule(lr){9-10} \cmidrule(lr){11-13}
    \bf Kernel & 
    $1\times$ & $2\times$ & $4\times$ & $8\times$ & $16\times$ & $32\times$ & $64\times$ &
    $1\times$ & $1.5\times$ &
    $1\times$ & $2\times$ & $4\times$ \\
    
    \midrule 
    
Transformer & \textbf{1.00}$^{(\underline{1.00})}$ &\textbf{1.00}$^{(\underline{1.00})}$&
0.99$^{(0.99})$& 0.09$^{(0.15})$&	 0.01$^{(0.01)}$ & 0.00$^{(0.00)}$& 0.00$^{(0.00)}$ & \textbf{1.00}$^{(\underline{1.00})}$ & \textbf{0.52}$^{(0.62)}$ & \textbf{1.00}$^{(\underline{1.00})}$ & 0.00$^{(0.00)}$ & 0.00$^{(0.00)}$  \\

Gaussian & 0.99$^{(\underline{1.00})}$ & 0.96$^{(\underline{1.00})}$ & 0.93$^{(0.99)}$ & 0.85$^{(0.97)}$ & 0.64$^{(0.81)}$ & 0.14$^{(0.20)}$ & 0.00$^{(0.00)}$ & \textbf{1.00}$^{(\underline{1.00})}$ & 0.14$^{(0.38)}$ & \textbf{1.00}$^{(\underline{1.00})}$ & 0.06$^{(0.17)}$ & 0.00$^{(0.00)}$ \\

Gaussian kNN, $k=32$ & 0.99$^{(\underline{1.00})}$ & 0.97$^{(\underline{1.00})}$ & 0.93$^{(0.99)}$ & 0.84$^{(0.96)}$ & 0.68$^{(0.85)}$ & 0.27$^{(0.36)}$ & 0.05$^{(0.05)}$ & 0.99$^{(\underline{1.00})}$ & 0.00$^{(0.00)}$ & \textbf{1.00}$^{(\underline{1.00})}$ & 0.80$^{(0.91)}$ & 0.01$^{(0.02)}$ \\

Uniform kNN, $k=32$ & 0.04$^{(0.08)}$ & 0.00$^{(0.00)}$ & 0.00$^{(0.00)}$ & 0.00$^{(0.00)}$ & 0.00$^{(0.00)}$ & 0.00$^{(0.00)}$ & 0.00$^{(0.00)}$ & 0.42$^{(0.66)}$ & 0.00$^{(0.00)}$ & \textbf{1.00}$^{(\underline{1.00})}$ & 0.00$^{(0.00)}$ & 0.00$^{(0.00)}$ \\

Epanechnikov (normReLU) & \textbf{1.00}$^{(\underline{1.00})}$ & \textbf{1.00}$^{(\underline{1.00})}$ & 0.89$^{(\underline{1.00})}$ & 0.66$^{(\underline{1.00})}$ & 0.44$^{(\underline{1.00})}$ & 0.33$^{(\underline{0.98})}$ & \textbf{0.31}$^{(\underline{0.94})}$ & \textbf{1.00}$^{(\underline{1.00})}$ & 0.19$^{(0.56)}$ & \textbf{1.00}$^{(\underline{1.00})}$ & 0.05$^{(0.14)}$ & 0.00$^{(0.00)}$ \\

Epanechnikov (1-ReLUmax) & \textbf{1.00}$^{(\underline{1.00})}$ & \textbf{1.00}$^{(\underline{1.00})}$ & 0.99$^{(\underline{1.00})}$ & 0.92$^{(0.96)}$ & 0.74$^{(0.87)}$ & 0.42$^{(0.59)}$ & 0.12$^{(0.18)}$ & \textbf{1.00}$^{(\underline{1.00})}$ & 0.23$^{(0.68)}$ & \textbf{1.00}$^{(\underline{1.00})}$ & 0.80$^{(\underline{1.00})}$ & 0.30$^{(\underline{0.91})}$ \\

Epanechnikov (sparsemax) & \textbf{1.00}$^{(\underline{1.00})}$ & \textbf{1.00}$^{(\underline{1.00})}$ & \textbf{1.00}$^{(\underline{1.00})}$ & \textbf{0.98}$^{(\underline{1.00})}$ & 0.86$^{(0.98)}$ & \textbf{0.63}$^{(0.84)}$ & \textbf{0.31}$^{(0.51)}$ & \textbf{1.00}$^{(\underline{1.00})}$ & 0.03$^{(0.06)}$ & \textbf{1.00}$^{(\underline{1.00})}$ & 0.92$^{(\underline{1.00})}$ & 0.29$^{(0.79)}$ \\

Biweight (2-ReLUmax) & 0.99$^{(1.00)}$ & 0.97$^{(1.00)}$ & 0.93$^{(1.00)}$ & 0.85$^{(0.96)}$ & 0.69$^{(0.87)}$ & 0.27$^{(0.43)}$ & 0.04$^{(0.06)}$ & \textbf{1.00}$^{(\underline{1.00})}$ & 0.26$^{(0.74)}$ & \textbf{1.00}$^{(\underline{1.00})}$ & 0.93$^{(\underline{1.00})}$ & \textbf{0.34}$^{(0.74)}$ \\

Biweight (1.5-entmax) & \textbf{1.00}$^{(\underline{1.00})}$ & \textbf{1.00}$^{(\underline{1.00})}$ & 0.99$^{(\underline{1.00})}$ & 0.95$^{(0.99)}$ & 0.75$^{(0.93)}$ & 0.51$^{(0.74)}$ & 0.23$^{(0.42)}$ & \textbf{1.00}$^{(\underline{1.00})}$ & 0.05$^{(0.06)}$ & \textbf{1.00}$^{(\underline{1.00})}$ & 0.39$^{(\underline{1.00})}$ & 0.26$^{(0.77)}$ \\

Triweight (3-ReLUmax) & 0.99$^{(\underline{1.00})}$ & 0.97$^{(\underline{1.00})}$ & 0.93$^{(0.99)}$ & 0.85$^{(0.96)}$ & 0.71$^{(0.87)}$ & 0.37$^{(0.58)}$ & 0.08$^{(0.18)}$ & \textbf{1.00}$^{(\underline{1.00})}$ & 0.29$^{(\underline{0.88})}$ & \textbf{1.00}$^{(\underline{1.00})}$ & 0.95$^{(0.97)}$ & 0.30$^{(0.62)}$ \\

Triweight (4/3-entmax) & \textbf{1.00}$^{(\underline{1.00})}$ & \textbf{1.00}$^{(\underline{1.00})}$ & 0.99$^{(\underline{1.00})}$ & 0.97$^{(0.98)}$ & \textbf{0.89}$^{(0.95)}$ & 0.61$^{(0.76)}$ & 0.25$^{(0.45)}$ & \textbf{1.00}$^{(\underline{1.00})}$ & 0.26$^{(0.70)}$ & \textbf{1.00}$^{(\underline{1.00})}$ & \textbf{0.98}$^{(0.99)}$ & 0.07$^{(0.11)}$ \\

    \bottomrule
  \end{tabular}
   }
\end{table*}
\section{Discussion and Related Work}\label{sec:related_work}

\paragraph{Sparse attention mechanisms.}
Sparse attention offers an alternative to dense softmax, enabling models to focus on a subset of relevant tokens.
Early approaches such as top-$k$ softmax~\citep{gupta-etal-2021-memory} enforced sparsity by keeping only the largest attention scores, but suffered from non-smooth gradients and a fixed attention budget. One line of work employs predefined sparse patterns, such as local sliding windows or fixed strides, such as LongFormer~\citep{beltagy_longformer_2020} and BigBird~\citep{zaheer2020bigbird}. 
Other methods introduce end-to-end learnable sparsity by replacing softmax with differentiable alternatives such as sparsemax~\citep{martins2016softmax} and its generalization $\alpha$-entmax~\citep{peters2019sparse}.
More recently, \cite{hu2023sparse},  \cite{wu2024stanhop}, \cite{santos2024sparse,santos2024hopfieldfenchelyoungnetworksunifiedframework}, and \cite{hoover2025dense} develop sparse Hopfield networks---mathematically equivalent to sparse attention mechanisms---, while \cite{vasylenko2025longcontextgeneralizationsparseattention} show that $\alpha$-entmax attention can mitigate attention dispersion and representational collapse for long sequences. Our paper provides a test-time regression interpretation of these frameworks. 

\paragraph{Kernel methods in transformers.}
Early interpretations view softmax attention as a kind of similarity kernel, where attention weights arise from pairwise comparisons between queries and keys~\citep{tsai2019TransformerDissection}.
Building on this statistical view, recent work has shown that at test time transformers can behave like implicit Nadaraya-Watson estimators, with attention patterns converging to regression‑like weights that resemble classical kernel regression outputs~\citep{wang2025testtime}. Related research explores test‑time regression and meta‑learning behaviours in transformers and fast weight systems where models implicitly solve regression or optimization tasks in context~\citep{von2023transformers,schlag2021linear,behrouz2024titans,sun2025learning}. While prior work focuses on dense kernels and efficient approximations, we extend the kernel regression view to sparse, compactly supported kernels, unifying dense and sparse attention within a single framework.

\paragraph{Neurobiological context.} %In the domain of neurobiological systems such as humans, 
In neurobiological systems, the problem of understanding the long-range dependency structure in sequences arises in episodic memory \citep{tulving2002episodic}. The hippocampal indexing theory suggests that the brain solves this problem with hippocampal engrams serving as indices with which the brain retrieves neocortical traces of previous experiences \citep{teyler1986hippocampal,goode2020integrated}. Notably, a core characteristic of hippocampal population activity is sparseness. This property of hippocampal activity is often attributed to its functional role of pattern separation \cite{yassa2011patternseparation}. From the perspective of the current work, hippocampal indices serve as keys, which can be queried, to neocortical values \cite{gershman2025keyvalue}. Thus, sparsity can be understood as an implicit use of compact-support kernels whereby neocortical memory retrieval is mediated by a relatively small set of well-separated memories, rather than weakly integrating over an increasingly large number of memories accumulated over a lifetime \cite{tulving2002episodic}.

\section{Conclusions}
\label{sec:conclusions}
In this work, we show that sparse attention mechanisms can be interpreted as kernel regression with compactly supported kernels, providing a principled framework that unifies dense and sparse attention. We demonstrate that $r$-order rectified polynomial kernels allow controlled sparsity, from which several attention mechanisms, including $\alpha$-entmax, can be derived. For $r=1$, the Epanechnikov kernel underlies several attention mechanisms: normalized ReLU, which can produce empty-support indeterminacies when all pre-activations are negative; 1-ReLUmax, which anchors the support at the maximum key to avoid this issue; and sparsemax, which further introduces adaptive normalization by adjusting the kernel bandwidth based on the attention scores also ensuring non-degenerate attention. Our empirical results confirm that attention mechanisms with adaptive normalization outperform dense Gaussian and heuristic top-$k$ methods, generally delivering superior performance. 
While we experiment with models of various sizes, our experiments are all with relatively small models ($<$1B parameters). Further scaling experiments, beyond our compute budget, are needed to assess the benefits of compact kernels in larger models. We highlight this as a direction for future work.

\section*{Acknowledgments}
We would like to thank Hugo Pitorro, Lili Mou, Pavlo Vasylenko, Mathias Lindemann, and the SARDINE lab team for the helpful discussions.
We thank the Champalimaud Foundation for supporting this research. 
This work was supported  by the project DECOLLAGE (ERC-2022-CoG 101088763), and by FCT/MECI
through national funds and when applicable co-funded EU funds under UID/50008: Instituto de
Telecomunicações.

\bibliographystyle{plain}
\bibliography{bib}
%%%%%%%%%%%%%%%%%%%%%%%%%%%%%%%%%%%%%%%%%%%%%%%%%%%%%%%%%%%%

\appendix
\onecolumn
\clearpage
\newpage

\section{Proof of Proposition~\ref{prop:entmax_rectified_poly}}
\label{proof:main_proof}

\begin{proof}
Let $\bm{q} \equiv \bm{k}_n$ and $r = \frac{1}{\alpha-1}$.
Since $\|\kk_i\| = \|\bm{q}\| = 1$, we have $\frac{1}{2}\|\kk_i - \bm{q}\|^2 = 1 - \kk_i^\top \bm{q}$. 
The $r$-weight polynomial kernel is defined as:
\begin{align}
K_h(\kk_i - \bm{q}) &= \left[1 - \frac{\|\kk_i - \bm{q}\|^2}{h^2}\right]_+^r \\
&= \left[1 - \frac{2(1 - \kk_i^\top \bm{q})}{h^2}\right]_+^r \\
&= \left[\frac{h^2 - 2 + 2\kk_i^\top \bm{q}}{h^2}\right]_+^r.
\end{align}
That is, the induced kernel takes the form $K_h(\bm{u})
\propto
\left[1 - \|\bm{u}\|^2\right]_+^r$ ,
a rectified polynomial kernel of order $r$, encompassing classical kernels such as the biweight ($r=2$) and triweight ($r=3$).
Now, introducing a temperature $\gamma$, the denominator in \eqref{eq:entmax_proposition_1} can be expressed as
\begin{align}
\sum_{i=1}^{n-1}\left[1 - \frac{\|\kk_i - \bm{q}\|^2}{h^2}\right]_+^r
&= \left(\frac{2r\gamma}{h^2}\right)^r
\sum_{i=1}^{n-1}
\left[\frac{h^2}{2r\gamma} - \frac{1}{r\gamma} + \frac{\kk_i^\top \bm{q} }{r\gamma}\right]_+^r. \label{eq:epanechnikov_rewrite}
\end{align}

Choosing $h$ such that
\begin{align}
\sum_{i=1}^{n-1}
\left[\frac{\kk_i^\top \bm{q}}{r\gamma} - \underbrace{\left(\frac{1}{r\gamma} - \frac{h^2}{2r\gamma}\right)}_{:= \tau}  \right]_+^r = 1, \label{eq:entmax_norm}
\end{align}
the normalizer $\tau = \frac{1}{r\gamma} - \frac{h^2}{2r\gamma}$ coincides with the $\alpha$-entmax threshold \eqref{eq:entmax}, yielding
\begin{align}
\hat{\vv}_n(\bm{q})
= \sum_{i=1}^{n-1}
\alpha\text{-}\text{entmax}_i\!\left(\frac{\K\bm{q}}{\gamma}\right)\vv_i,
\label{eq:entmax_kernel_regression}
\end{align}
with an adaptive bandwidth $h = \sqrt{2 - 2r\gamma \tau}$.
\end{proof}

% \newpage

\section{Experimental Details}
\label{app:experimental_details}
\subsection{Language Modeling}
\label{subsec:app_lm}

In this section, we provide additional details of our experimental setup, including model configurations and training hyperparameters for the language modeling task (Table~\ref{tab:training_hyperparams}). We also show the full training and validation curves for the BabiStories dataset (see statistics in Table~\ref{tab:babistories_stats}) \citep{zhang2025memory}, complementing the results in the main text (Figure~\ref{fig:train_val_full}). All models use the GPT-2 tokenizer \citep{radford2019language-gpt2} and share a sequence length of 512 tokens. 
Experiments cover different numbers of parallel layers of persistent and contextual memories. For single-layer models, we vary $k \in {16, 32, 64}$ and report results for the best-performing value ($k=32$), which is then used for all multi-layer configurations. Each contextual memory unit computes leaky-average keys and one-step-ahead values, with retrieval performed via kernel-weighted interpolation over past embeddings. All models are trained for 80k iterations with a global batch size of 512, optimized with Adam, and using learning rate schedules consistent with GPT-2 Small \citep{radford2019language-gpt2}. For 12-layer models, the learning rate is reduced to $10^{-3}$. Sparse attention patterns for the rectified polynomial kernels are learned efficiently using AdaSplash \citep{goncalves2025adasplash}. While in practice the normalized ReLU is not often met, we use uniform weighting for those cases. These experiments were run on h100 and h200 GPUs.
\begin{table}[t]
\centering
\caption{Main training and model hyperparameters across model depths. 
Unless otherwise specified, hyperparameters are shared across all experiments.$^*$ means that for 2-ReLUmax we used a learning rate of $1 \times 10^{-3}$ and minimum learning rate of $5 \times 10^{-5}$ due to training instability.}
\label{tab:training_hyperparams}
\begin{tabular}{lccc}
\toprule
\textbf{Hyperparameter} & \textbf{1 layer} & \textbf{8 layers} & \textbf{12 layers} \\
\midrule
\multicolumn{4}{l}{\emph{Optimization}} \\
Batch size (total)        & 512 & 512 & 512 \\
Learning rate             & $5{\times}10^{-3} $& $5{\times}10^{-3}$ $^*$ & $1{\times}10^{-3}$ \\
Min learning rate         & $1{\times}10^{-4}$ & $1{\times}10^{-}4$ $^*$  & $5{\times}10^{-5}$ \\
Max iterations            & 80k & 80k & 80k \\
Weight decay              & 0.1 & 0.1 & 0.1 \\
Adam $\beta_1$ / $\beta_2$ & 0.9 / 0.95 & 0.9 / 0.95 & 0.9 / 0.95 \\
Gradient clipping         & 1.0 & 1.0 & 1.0 \\
\midrule
\multicolumn{4}{l}{\emph{Learning rate schedule}} \\
Scheduler                 & \multicolumn{3}{c}{Cosine decay} \\
Warmup iterations         & \multicolumn{3}{c}{2k} \\
LR decay iterations       & \multicolumn{3}{c}{80k} \\
\midrule
\multicolumn{4}{l}{\emph{Model architecture}} \\
Context length ($n$)      & \multicolumn{3}{c}{512} \\
Embedding dimension ($d$) & \multicolumn{3}{c}{768} \\
Attention heads           & \multicolumn{3}{c}{12} \\

\midrule
\multicolumn{4}{l}{\emph{Persistent memory}} \\
Memory size               & \multicolumn{3}{c}{2688} \\
\bottomrule
\end{tabular}
\end{table}

\begin{table}[ht!]
\centering
\caption{BabiStories dataset statistics \citep{zhang2025memory}.}
\label{tab:babistories_stats}
\begin{tabular}{lccc}
\toprule
Dataset partition & \# Stories & \# Tokens (GPT-2 tokenizer) & Avg. \# Characters per story \\
\midrule
Train & 2.2M  & 474,704,907 & 888 \\
Valid & 2.2k  & 4,749,107   & 889 \\
\bottomrule
\end{tabular}
\end{table}

\begin{figure*}[ht!]
    \centering
    \includegraphics[width=1\linewidth]{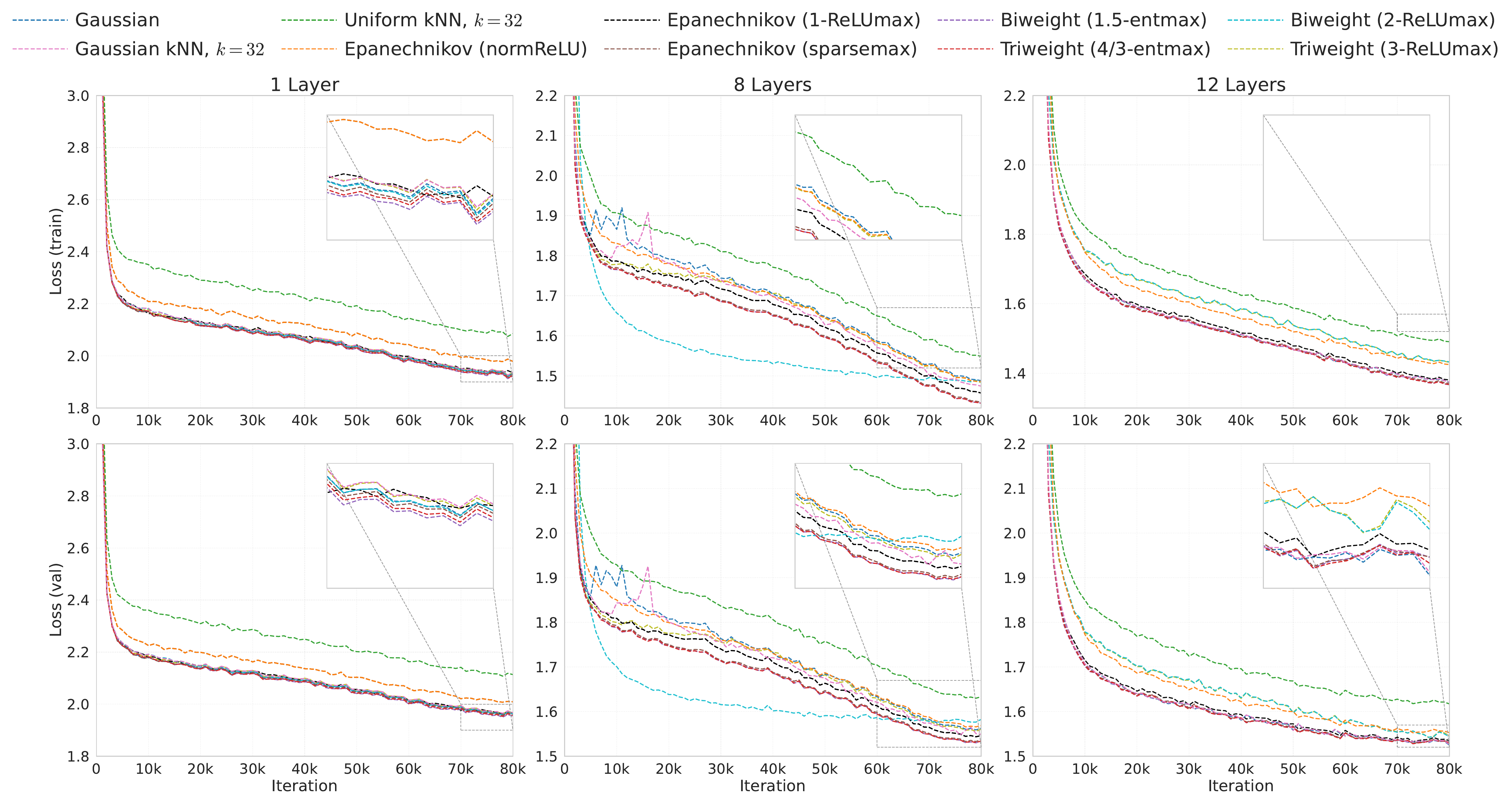}
    \caption{\textbf{Training and validation loss of Memory Mosaic models with different kernels.} Comparison on the BabiStories dataset for varying model depths. The horizontal axis shows the number of training iterations.}

    \label{fig:train_val_full}
\end{figure*}

\subsection{In-Context Learning}
\label{app:icl_experimental_details}

For evaluation, we use RegBench~\citep{akyürek2024incontextlanguagelearningarchitectures}.
RegBench constructs random artificial languages defined by a PFA. Each input sequence consists of 10–20 strings sampled from the same PFA, separated by special delimiter tokens. Models are evaluated on their ability to predict the last token of testing sequences generated from held-out PFAs, effectively measuring their capacity to generalize the underlying generative rules.

For Memory Mosaics, we perform a grid search over the Gaussian kernel hyperparameters, select the best configuration based on validation loss, and apply those hyperparameters to the remaining kernels. 
During the hyperparameter search for the Gaussian kernel variant, various configurations were tested, including different number of blocks and weight decay. Particular attention was given to the hyperparameters of the Memory Mosaics, such as the number of heads and embedding size.  
For kNN-based approaches, we tune $k \in \{16, 32, 64\}$ on 1k training examples and use that value ($k=16$) for the remaining training datasets and for uniform kNN. 
All hyperparameters tested are shown in Table \ref{tab:icl_hyperparam}. These experiments were run on one single a6000 GPU.
\begin{table}[h]
    \centering
    \caption{Hyperparameter space for the in-context learning experiment. Depth denotes the number of stacked blocks comprising persistent and contextual memories. Embedding size corresponds to the head dimension in transformers. We report results from the epoch achieving the lowest validation loss.}
    \label{tab:icl_hyperparam}
    \begin{tabular}{ l  l }
        \toprule
        {\textbf{Parameter}} & {\textbf{Range}} \\
        \midrule
        weight decay & \{$10^{-1}$, $10^{-2}$\} \\
        depth & \{2, 4, 8\}  \\
        embedding size & \{32, 128, 256\}  \\
        number of heads & \{1, 2, 4, 8\}  \\
        learning rate & \{$5\times 10^{-4}$\} \\
        maximum number of epochs & \{200\} \\
        batch size & \{32\}\\
        \bottomrule
    \end{tabular}
\end{table}

\subsection{Length Generalization}
\label{app:synthetic_experimental_details}
For biological agents, episodic memory scales with the length of life rather than a fixed context window, placing memory retrieval inherently in the long-context regime and requiring continual length generalization. In such settings, dense aggregation over all stored memories would lead to severe interference, motivating the use of sparse retrieval mechanisms that selectively engage only a small subset of relevant memories. This perspective aligns with hippocampal indexing theory, where retrieval operates over sparse indices rather than dense memory traces \citep{teyler1986hippocampal,goode2020integrated}. Inspired by this neurobiological constraint, we evaluate sparse attention mechanisms in a set of controlled synthetic tasks designed to isolate core sequence-modeling abilities. Specifically, we consider multi-query multi-token associative recall (MQMTAR), sequence reversal, and sequence sorting, which respectively probe long-range associative retrieval, sequential manipulation, and global reordering. To ensure comparability across tasks and focus on the effects of the attention mechanism itself, we fix all model hyperparameters to modest values rather than performing extensive task-specific tuning; the shared hyperparameters are summarized in Table~\ref{tab:combined_hyperparam}. We show in Table~\ref{tab:k_results} additional results for different $k$ for the kNN variants. For the Transformer model, we use the GPT2 \citep{radford2019language-gpt2} architecture with the ALiBi \cite{press2022train} positional embedding. These experiments were run on A6000 GPUs.

\paragraph{Multi-query Multi-token Associative Recall (MQMTAR).}
The Multi-query Multi-token Associative Recall task evaluates a model’s ability to store and retrieve multiple key–value associations within a single sequence and to answer several queries in parallel. The vocabulary consists of 256 integers, augmented with utility tokens \{0, 1, 2\} and additional special tokens used for structural delimitation. Each key $k_i$, value $v_i$, and query $q_j$ is represented by a fixed-length sequence of two discrete tokens, corresponding to two numbers drawn from the base vocabulary. During input construction, multiple key–value pairs $(k_i, v_i)$ are presented sequentially, followed by a set of four query keys $\{q_1, q_2, q_3, q_4\}$, all sharing the same two-token structure as the keys seen during the context. The model’s objective is to retrieve, for each query $q_j$, the value $v_i$ associated with the matching key $k_i$ encountered earlier in the sequence. The target output is formed by concatenating the two-token representations of the four retrieved values in query order. This task jointly probes multi-token associative binding, robustness to interference across multiple stored associations, and the ability to perform accurate multi-query recall within a single forward pass.

\begin{figure}[h!]
    \centering
    \includegraphics[width=1\linewidth]{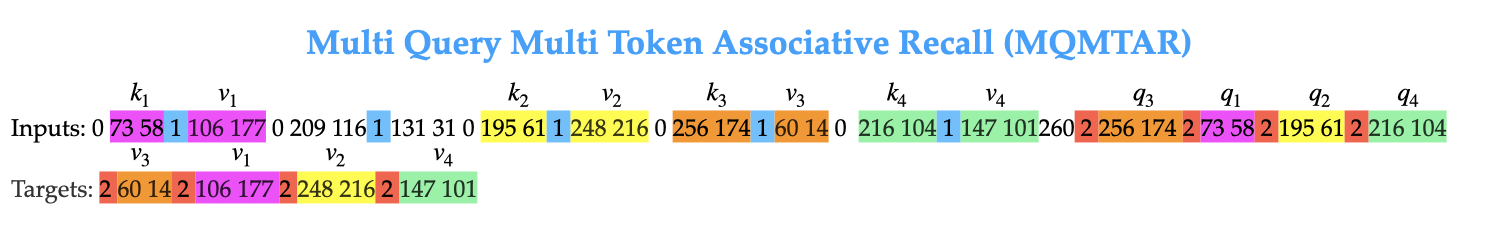}
    \includegraphics[width=0.85\linewidth]{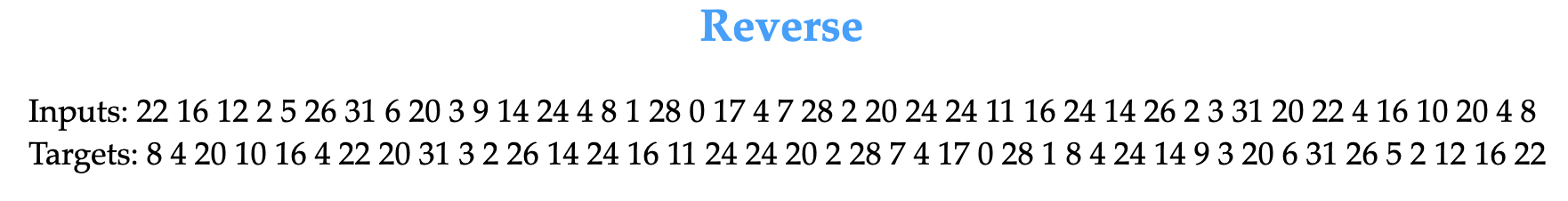}
    \includegraphics[width=0.6\linewidth]{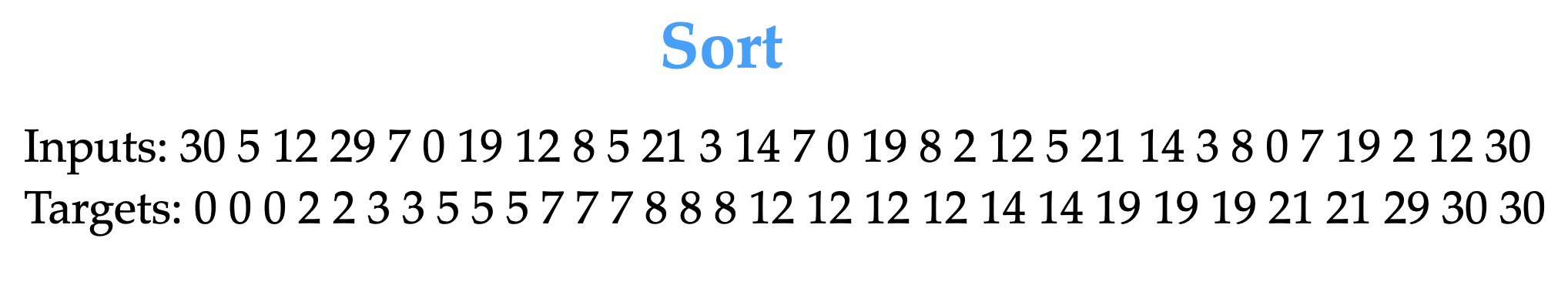}
    \caption{Illustration of the MQMTAR, reverse, and sort tasks. 
    In the MQMTAR task, Soft colors indicate different keys, values, and queries for visual clarity. 
    In the reverse task, the model receives a sequence of integers and must output them in reverse order.
    In the sort task, the model receives a sequence of integers and must output the sequence in ascending order.
    }

    \label{fig:all_tasks}
\end{figure}
\begin{table}[t]
\centering
\caption{Hyperparameters used for MQMTAR recall, reverse, and sorting tasks. Hyperparameters were fixed to keep models simple and comparable.}
\label{tab:combined_hyperparam}
\begin{tabular}{lccc}
\toprule
\textbf{Parameter} & \textbf{MQMTAR} & \textbf{Reverse} & \textbf{Sorting} \\
\midrule
Training examples & $50\times 10^{6}$ & $30\times 10^{6}$ & $40\times 10^{6}$ \\
\# epochs & 1&1&1\\
Vocab size & 263 & 36 & 36 \\
Weight decay & $10^{-1}$ & $10^{-1}$ & $10^{-1}$ \\
Depth & 4 & 4 & 2 \\
Embedding size & 512 & 512 & 512 \\
Number of heads & 8 & 8 & 8 \\
Learning rate & $10^{-3}$ & $10^{-3}$ & $10^{-3}$ \\
Batch size & 256 & 256 & 256 \\
Training sequence length &  up to 64& 64 & 64 \\
\bottomrule
\end{tabular}
\end{table}

\paragraph{Reverse.} The reverse task evaluates a model’s ability to store and reproduce the elements of a discrete sequence in reverse order. The vocabulary consists of integers from 0 to 31, and additional special tokens used for structural delimitation.

During input construction, a sequence of integers $(x_1, x_2, \dots, x_n)$ is presented to the model. The model’s objective is to output the sequence in reverse order $(x_n, x_{n-1}, \dots, x_1)$. This task probes the model’s ability to memorize and manipulate sequential information over short-to-moderate context lengths.

\paragraph{Sort.} The sort task evaluates a model’s ability to correctly reorder a sequence of discrete tokens in ascending order. The vocabulary consists of integers from 0 to 31, including repetitions.

During input construction, a sequence of integers $(x_1, x_2, \dots, x_n)$ is presented to the model. The model’s objective is to output the sequence sorted in non-decreasing order $(y_1 \le y_2 \le \dots \le y_n)$. This task probes the model’s ability to maintain and manipulate sequential information over short-to-moderate context lengths.

\begin{table*}[t]
\caption{Exact match accuracy on representative synthetic tasks. In-distribution results ($n=64$) and out-of-distribution performance at increasing sequence lengths are reported. Values show the mean across three seeds, with the maximum across seeds indicated in superscript. Best average performance is in \textbf{bold} and overall maximum performance is \underline{underlined}. $L$ represents the number of layers.}
  \small
  \label{tab:k_results}
  \centering
   \setlength{\tabcolsep}{2pt}
   \resizebox{\textwidth}{!}{
  \begin{tabular}{l ccccccc cc ccc}
    \toprule
    & \multicolumn{7}{c}{MQMTAR $(L=4)$} 
    & \multicolumn{2}{c}{Reverse $(L=4)$}
    & \multicolumn{3}{c}{Sort $(L=2)$} \\
    \cmidrule(lr){2-8} \cmidrule(lr){9-10} \cmidrule(lr){11-13}
    \bf Kernel & 
    $1\times$ & $2\times$ & $4\times$ & $8\times$ & $16\times$ & $32\times$ & $64\times$ &
    $1\times$ & $1.5\times$ &
    $1\times$ & $2\times$ & $4\times$ \\
    
    \midrule 
Gaussian kNN, $k=16$ & \textbf{0.99}$^{(\underline{1.00})}$ & \textbf{0.97}$^{(\underline{1.00})}$ & \textbf{0.93}$^{(\underline{0.99})}$ & \textbf{0.85}$^{(\underline{0.96})}$ & 0.65$^{(0.78)}$ & 0.25$^{(\underline{0.36})}$ & \textbf{0.05}$^{(\underline{0.09})}$ & 0.86$^{(\underline{1.00})}$ & \textbf{0.00}$^{(\underline{0.00})}$ & \textbf{1.00}$^{(\underline{1.00})}$ & 0.74$^{(0.89})$ & 0.00$^{(0.00)}$ \\

Gaussian kNN, $k=32$ & \textbf{0.99}$^{(\underline{1.00})}$ & \textbf{0.97}$^{(\underline{1.00})}$ & \textbf{0.93}$^{(\underline{0.99})}$ &   0.84$^{(\underline{0.96})}$ & \textbf{0.68}$^{(\underline{0.85})}$ & \textbf{0.27}$^{(\underline{0.36})}$ & \textbf{0.05}$^{(0.05)}$ & \textbf{0.99}$^{(\underline{1.00})}$ & \textbf{0.00}$^{(\underline{0.00})}$ & \textbf{1.00}$^{(\underline{1.00})}$ & \textbf{0.80}$^{(\underline{0.91})}$ & \textbf{0.01}$^{(\underline{0.02})}$ \\

\midrule 
Uniform kNN, $k=16$  
& \textbf{0.58}$^{(\underline{0.71})}$ 
& \textbf{0.01}$^{(\underline{0.01})}$ 
& \textbf{0.00}$^{(\underline{0.00})}$ 
& \textbf{0.00}$^{(\underline{0.00})}$ 
& \textbf{0.00}$^{(\underline{0.00})}$ 
& \textbf{0.00}$^{(\underline{0.00})}$ 
& \textbf{0.00}$^{(\underline{0.00})}$ 
& 0.00$^{(0.00})$ 
& \textbf{0.00}$^{(\underline{0.00})}$ 
& \textbf{1.00}$^{(\underline{1.00})}$ 
& \textbf{0.41}$^{(\underline{0.63})}$ 
& \textbf{0.00}$^{(\underline{0.00})}$ \\

Uniform kNN, $k=32$ 
& 0.04$^{(0.08)}$ 
& \textbf{0.00}$^{(\underline{0.00})}$ 
& 0.00$^{(.00})$ 
& \textbf{0.00}$^{(\underline{0.00})}$ 
& \textbf{0.00}$^{(\underline{0.00})}$ 
& \textbf{0.00}$^{(\underline{0.00})}$ 
& \textbf{0.00}$^{(\underline{0.00})}$ 
& \textbf{0.42}$^{(\underline{0.66})}$ 
& \textbf{0.00}$^{(\underline{0.00})}$ 
& \textbf{1.00}$^{(\underline{1.00})}$ 
&0.00$^{(0.00})$ 
& \textbf{0.00}$^{(\underline{0.00})}$ \\

    \bottomrule
  \end{tabular}
   }
\end{table*}
\subsection{Efficiency Considerations}
\label{app:efficiency}
In standard dense softmax attention, each query distributes probability mass over all keys, resulting in computation and memory costs that scale linearly with the sequence length. In contrast, compact-support kernels and sparse attention transformations can assign exactly zero weight to many positions, meaning that only a subset of tokens participates in the weighted aggregation. This introduces the potential for reduced computational and memory overhead, particularly in long-context settings.

We note that the experiments in this paper operate in a short-context regime ($512$ tokens for the language modeling and up to a few thousands for the synthetic tasks) in which attention is not the dominant cost and the wall-clock gains of any fused attention kernel are modest~\citep{dao2023flashattention2}. Our focus is therefore on the modeling implications of compact-support kernels rather than on raw throughput.

At longer context lengths, however, the extent to which these theoretical benefits translate into practical speedups depends strongly on implementation. %Efficient kernelized attention mechanisms that explicitly exploit sparsity structure can achieve substantial gains. For instance, AdaSplash~\cite{goncalves2025adasplash}, a Triton-based implementation that we use in our experiments for entmax attention, can outperform FlashAttention-2 \citep{dao2023flashattention2} in regimes with high sparsity during training.
By the kernel regression correspondence established in this paper, the entire $\alpha$-entmax family (e.g. Epanechnikov, biweight, triweight, ...) are fully compatible with  Adasplash~\citep{goncalves2025adasplash}, which can match and even outperform FlashAttention~\citep{dao2023flashattention2} in high-sparsity regimes which naturally emerge at longer context lengths.

By contrast, combinatorial sparsification mechanisms such as top-$k$ attention often require explicit selection or sorting operations, which are difficult to fuse into optimized GPU kernels and may limit practical efficiency despite inducing sparsity in the attention pattern.

%Overall, this highlights that efficiency is determined not only by the mathematical form of the attention mechanism but also by how sparsity is implemented in practice. Among the mechanism here, those expressible as $\alpha$-entmax admit existing fused implementation, while combinatorial schemes such as top-$k$ remain harder to fuse efficiently.

\section{Additional Experiments}
\label{app:experiments}
\subsection{Sparse Memory Mosaics with Positional Embedding}
\label{app:alibi}

To assess whether our gains stem from the proposed method rather than existing sparsity-inducing mechanisms, we compare against a standard Transformer with ALiBi \citep{press2022train}, which introduces a positional bias favoring local attention.

\begin{table*}[t]
\caption{Exact match accuracy on the MQMTAR task. In-distribution results ($1\times$) and out-of-distribution performance at increasing sequence lengths are reported. Values show the mean across seeds, with the maximum indicated in superscript. Best average performance is in \textbf{bold} and overall maximum performance is \underline{underlined}.}
\small
\label{tab:alibi_comparison}
\centering
\setlength{\tabcolsep}{4pt}
\resizebox{\textwidth}{!}{
\begin{tabular}{l ccccccc}
\toprule
 & \multicolumn{7}{c}{MQMTAR} \\
\cmidrule(lr){2-8}
\bf Model 
& $1\times$ & $2\times$ & $4\times$ & $8\times$ & $16\times$ & $32\times$ & $64\times$ \\
\midrule

Transformer (ALiBi) 
& \textbf{1.00}$^{(\underline{1.00})}$ 
& \textbf{1.00}$^{(\underline{1.00})}$ 
& 0.99$^{(0.99)}$ 
& 0.09$^{(0.15)}$ 
& 0.01$^{(0.01)}$ 
& 0.00$^{(0.00)}$ 
& 0.00$^{(0.00)}$ \\

MM, Epanechnikov (ALiBi) 
& 0.99$^{(\underline{1.00})}$ 
& 0.96$^{(\underline{1.00})}$ 
& 0.92$^{(\underline{1.00})}$ 
& 0.85$^{(\underline{1.00})}$ 
& 0.72$^{(0.94)}$ 
& 0.43$^{(0.71)}$ 
& 0.14$^{(0.29)}$ \\

MM, Epanechnikov (no ALiBi) 
& \textbf{1.00}$^{(\underline{1.00})}$ 
& \textbf{1.00}$^{(\underline{1.00})}$ 
& \textbf{1.00}$^{(\underline{1.00})}$ 
& \textbf{0.98}$^{(\underline{1.00})}$ 
& \textbf{0.86}$^{(\underline{0.98})}$ 
& \textbf{0.63}$^{(\underline{0.84})}$ 
& \textbf{0.31}$^{(\underline{0.51})}$ \\

\bottomrule
\end{tabular}
}
\end{table*}

The results shown in Table~\ref{tab:alibi_comparison} show that, while ALiBi introduces a locality bias that encourages sparsity by prioritizing nearby tokens, it does not yield comparable performance. In contrast, our approach employs adaptive compact kernels to enable data-dependent sparsity, allowing the model to attend to relevant tokens irrespective of distance. As a result, Memory Mosaics maintains strong performance as context length increases, even without positional embedding, whereas the Transformer with ALiBi degrades rapidly.

\subsection{Out-Of-Distribution}
\paragraph{Out-of-distribution generalization on Simple English Wikipedia.}
Figure~\ref{fig:OOD} compares the prediction performance of Memory Mosaics models trained on the BabiStories corpus and evaluated out-of-distribution on Simple English Wikipedia\footnote{\url{https://huggingface.co/datasets/rahular/simple-wikipedia}}. 
The reported curves correspond to an exponential moving average (EMA) of the token-wise cross-entropy loss, computed over previously observed tokens, which provides a smoothed estimate of predictive confidence as the sequence unfolds.

At the beginning of each sequence, the model has limited contextual information and operates in a high-uncertainty regime. In this early phase, all attention variants exhibit similar behavior, with rapidly decreasing loss as additional context is incorporated, although higher-order compact kernels such as biweight and triweight typically start from slightly lower loss values. This reflects that, under uncertainty, strongly localized kernels may initially outperform dense Gaussian attention, which aggregates information more broadly.

As the sequence progresses and the model becomes more confident, clear differences emerge. Attention mechanisms based on compactly supported kernels—Epanechnikov, biweight, and triweight—achieve loss comparable to Gaussian attention, indicating that localized, sparse aggregation promotes selective, context-relevant memory retrieval in out-of-distribution settings. Among these, ReLUmax performs best overall: by combining sparse support with max anchoring, it avoids empty-support degeneracies while retaining locality, yielding stable and robust generalization as contextual evidence accumulates.

\begin{figure}[t]
    \centering
    \includegraphics[width=0.75\linewidth]{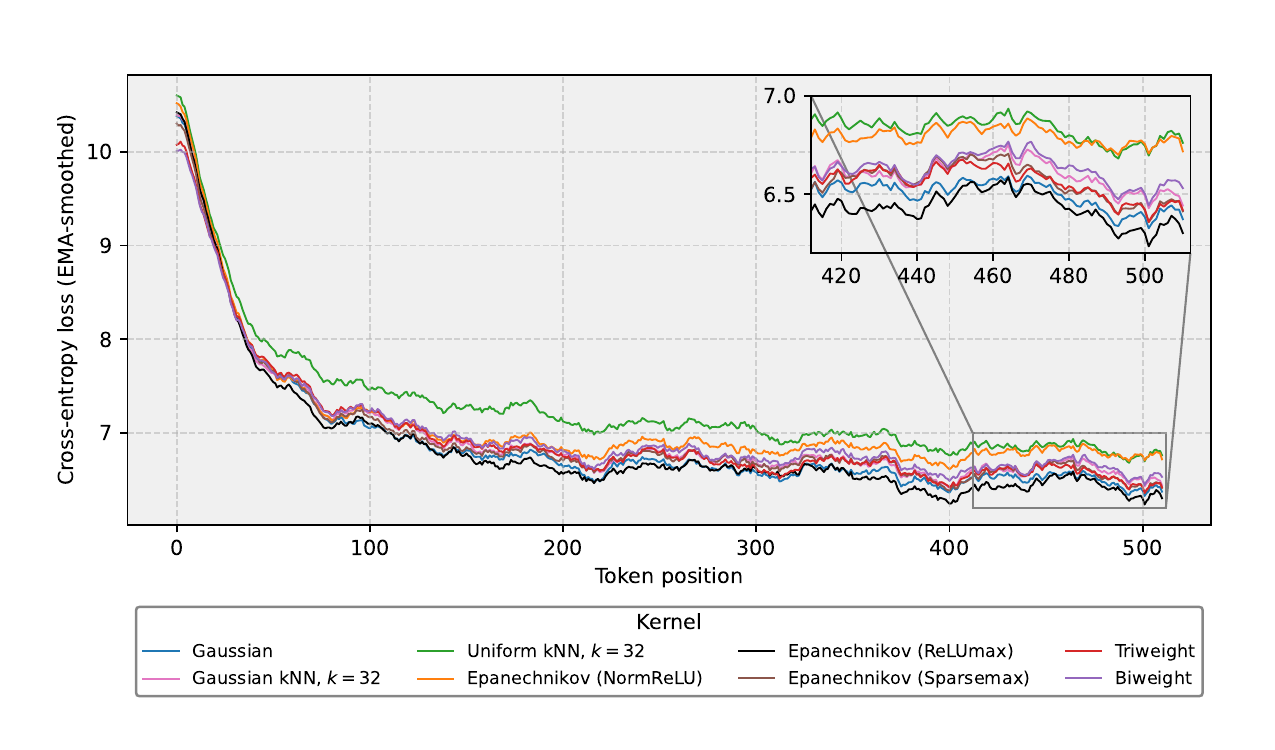}
    \caption{Prediction performance on the Simple English Wikipedia dataset using models trained on the BABISTORIES corpus. 
The plot reports the per-token average loss as a function of the token position within a 512-token input sequence.}
    \label{fig:OOD}
\end{figure}

\section*{Impact Statement}
\label{app:broader_impact}
This work advances the understanding of sparse attention mechanisms by framing them as kernel regression with compactly supported kernels. 
While we do not foresee ethical complications as a direct result of our method, we note that our contributions can be applied for training language models, and thus they inherit the broader ethical considerations associated with transformer-based language models. 
We emphasize responsible deployment, careful dataset curation, and energy-efficient training practices in such cases. 
%%%%%%%%%%%%%%%%%%%%%%%%%%%%%%%%%%%%%%%%%%%%%%%%%%%%%%%%%%%%

\newpage

\section*{NeurIPS Paper Checklist}

%%% END INSTRUCTIONS %%%

\begin{enumerate}

\item {\bf Claims}
    \item[] Question: Do the main claims made in the abstract and introduction accurately reflect the paper's contributions and scope?
    \item[] Answer: \answerYes{}
    \item[] Justification: The abstract and introduction clearly and accurately state the paper’s main contributions. The core claim, that sparse attention mechanisms can be understood as kernel regression with compact-support kernels, is explicitly presented in the abstract and then elaborated in the introduction. The introduction further lists concrete contributions (e.g., connections between sparsemax/$\alpha$-entmax and Epanechnikov-family kernels, the unification of top-k methods, and the proposal of ReLUmax), which align with the theoretical developments in \S\ref{sec:sparse_attention} and the proofs provided later (App.~\ref{proof:main_proof}).

\item {\bf Limitations}
    \item[] Question: Does the paper discuss the limitations of the work performed by the authors?
    \item[] Answer: \answerYes{}.
    \item[] Justification: While the paper does not include a dedicated limitations section, we note drawbacks and limitations of specific methods, such as the instability of top-k attention in \S~\ref{sec:language_modeling} and the degeneracy issues in normalized ReLU in \S\ref{sec:sparse_attention}. We also discuss the limitations in the conclusion (\S\ref{sec:conclusions}), where we point out that while we experiment with models of various sizes, our experiments are all with relatively small models ($<$1B parameters). Further scaling experiments, beyond our compute budget, are needed to assess the benefits of compact kernels in larger models. 

\item {\bf Theory assumptions and proofs}
    \item[] Question: For each theoretical result, does the paper provide the full set of assumptions and a complete (and correct) proof?
    \item[] Answer: \answerYes{}{} % Replace by \answerYes{}, \answerNo{}, or \answerNA{}.
    \item[] Justification: The paper presents its main theoretical results (e.g., Proposition~\ref{eq:entmax_proposition_1} linking $\alpha$-entmax to compact-support kernels) with clearly stated assumptions such as normalized keys/queries and kernel forms. The formal proof is provided in App.~\ref{proof:main_proof}, with the derivation steps shown in detail and consistent with the stated result. 

    \item {\bf Experimental result reproducibility}
    \item[] Question: Does the paper fully disclose all the information needed to reproduce the main experimental results of the paper to the extent that it affects the main claims and/or conclusions of the paper (regardless of whether the code and data are provided or not)?
    \item[] Answer: \answerYes{}{} % Replace by \answerYes{}, \answerNo{}, or \answerNA{}.
    \item[] Justification: The paper provides substantial details required for reproducibility. It specifies model architectures (Memory Mosaics), training procedures, hyperparameters (e.g., batch size, learning rates, optimizer settings), and datasets (e.g., BabiStories, RegBench, synthetic tasks), along with references. Additional implementation details are included in the appendix (\S\ref{app:experimental_details}) (e.g., training schedules, hyperparameter grids, and evaluation protocols), enabling a practitioner to replicate the experiments. While the code has not yet been released (promised upon acceptance), the methodological and experimental descriptions are sufficiently detailed to reproduce the main results.

\item {\bf Open access to data and code}
    \item[] Question: Does the paper provide open access to the data and code, with sufficient instructions to faithfully reproduce the main experimental results, as described in supplemental material?
    \item[] Answer: \answerNo{} % Replace by \answerYes{}, \answerNo{}, or \answerNA{}.
    \item[] Justification:  The code will be made available upon acceptance. 

\item {\bf Experimental setting/details}
    \item[] Question: Does the paper specify all the training and test details (e.g., data splits, hyperparameters, how they were chosen, type of optimizer) necessary to understand the results?
    \item[] Answer: \answerYes{}{} % Replace by \answerYes{}, \answerNo{}, or \answerNA{}.
    \item[] Justification: The paper provides all the necessary details for reproducibility. It specifies model architectures (Memory Mosaics), training procedures, hyperparameters (e.g., batch size, learning rates, optimizer settings), and datasets (e.g., BabiStories, RegBench, synthetic tasks), along with references and descriptions of how these datasets are constructed. Additional implementation details are provided in the appendix (\S\ref{app:experimental_details}) (e.g., training schedules, hyperparameter grids, and evaluation protocols), enabling practitioners to replicate the experiments. While the code has not yet been released (it will upon acceptance), the methodological and experimental descriptions are sufficiently detailed to reproduce the main results.

\item {\bf Experiment statistical significance}
    \item[] Question: Does the paper report error bars suitably and correctly defined or other appropriate information about the statistical significance of the experiments?
    \item[] Answer: \answerNo{} % Replace by \answerYes{}, \answerNo{}, or \answerNA{}.
    \item[] Justification: While the paper does not include error bars, confidence intervals, or statistical significance tests, we ran experiments with multiple seeds. Table \ref{tab:LG_results} reports averages over three seeds for all methods and reports the best results of each method across these seeds. 

\item {\bf Experiments compute resources}
    \item[] Question: For each experiment, does the paper provide sufficient information on the computer resources (type of compute workers, memory, time of execution) needed to reproduce the experiments?
    \item[] Answer: \answerYes{}{}{} % Replace by \answerYes{}, \answerNo{}, or \answerNA{}.
    \item[] Justification: We detail in the Appendix the type of compute workers used for each experiment. We further specify the global batch size for each experiment and discuss efficiency in \S\ref{app:efficiency} of all methods, as evaluating training times depends on several external factors.
    
\item {\bf Code of ethics}
    \item[] Question: Does the research conducted in the paper conform, in every respect, with the NeurIPS Code of Ethics \url{https://neurips.cc/public/EthicsGuidelines}?
    \item[] Answer: \answerYes{} % Replace by \answerYes{}, \answerNo{}, or \answerNA{}.
    \item[] Justification: 

\item {\bf Broader impacts}
    \item[] Question: Does the paper discuss both potential positive societal impacts and negative societal impacts of the work performed?
    \item[] Answer: \answerYes{} % Replace by \answerYes{}, \answerNo{}, or \answerNA{}.
    \item[] Justification: We present a broader impact section in \S\ref{app:broader_impact}.
    
\item {\bf Safeguards}
    \item[] Question: Does the paper describe safeguards that have been put in place for responsible release of data or models that have a high risk for misuse (e.g., pre-trained language models, image generators, or scraped datasets)?
    \item[] Answer: \answerNA{} % Replace by \answerYes{}, \answerNo{}, or \answerNA{}.
    \item[] Justification: 

\item {\bf Licenses for existing assets}
    \item[] Question: Are the creators or original owners of assets (e.g., code, data, models), used in the paper, properly credited and are the license and terms of use explicitly mentioned and properly respected?
    \item[] Answer: \answerYes{} % Replace by \answerYes{}, \answerNo{}, or \answerNA{}.
    \item[] Justification: We cite all the papers on which we based our code, and in the future release of the code (upon acceptance), we cite the GitHub repositories on which we based our code.

\item {\bf New assets}
    \item[] Question: Are new assets introduced in the paper well documented and is the documentation provided alongside the assets?
    \item[] Answer: \answerNA{} % Replace by \answerYes{}, \answerNo{}, or \answerNA{}.
    \item[] Justification: 

\item {\bf Crowdsourcing and research with human subjects}
    \item[] Question: For crowdsourcing experiments and research with human subjects, does the paper include the full text of instructions given to participants and screenshots, if applicable, as well as details about compensation (if any)? 
    \item[] Answer: \answerNA{} % Replace by \answerYes{}, \answerNo{}, or \answerNA{}.
    \item[] Justification: 

\item {\bf Institutional review board (IRB) approvals or equivalent for research with human subjects}
    \item[] Question: Does the paper describe potential risks incurred by study participants, whether such risks were disclosed to the subjects, and whether Institutional Review Board (IRB) approvals (or an equivalent approval/review based on the requirements of your country or institution) were obtained?
    \item[] Answer: \answerNA{} % Replace by \answerYes{}, \answerNo{}, or \answerNA{}.
    \item[] Justification: 

\item {\bf Declaration of LLM usage}
    \item[] Question: Does the paper describe the usage of LLMs if it is an important, original, or non-standard component of the core methods in this research? Note that if the LLM is used only for writing, editing, or formatting purposes and does \emph{not} impact the core methodology, scientific rigor, or originality of the research, declaration is not required.
    %this research? 
    \item[] Answer: \answerNo{} % Replace by \answerYes{}, \answerNo{}, or \answerNA{}.
    \item[] Justification: LLMs are employed solely for tasks such as writing, editing, or formatting and does \emph{not} affect the fundamental methodology, scientific rigor, or originality of the research.

\end{enumerate}

\end{document}